\documentclass{article}

\usepackage{microtype}
\usepackage{graphicx}
\usepackage{booktabs} 

\usepackage{hyperref}

\usepackage{customdefs}
\usepackage{natbib}

\usepackage[accepted]{icml2021arxiv}

\icmltitlerunning{Mediated Uncoupled Learning}

\begin{document}

\twocolumn[
\icmltitle{Mediated Uncoupled Learning:\\Learning Functions without Direct Input-output Correspondences}



\icmlsetsymbol{equal}{*}

\begin{icmlauthorlist}
\icmlauthor{Ikko Yamane}{dauphine,riken}
\icmlauthor{Junya Honda}{kyoto,riken}
\icmlauthor{Florian Yger}{dauphine,riken}
\icmlauthor{Masashi Sugiyama}{riken,tokyo}
\end{icmlauthorlist}

\icmlaffiliation{dauphine}{LAMSADE, CNRS, Universit{\' e} Paris-Dauphine, PSL Research University, 75016 PARIS, FRANCE}
\icmlaffiliation{riken}{RIKEN AIP, Tokyo, Japan}
\icmlaffiliation{kyoto}{Kyoto University, Kyoto, Japan}
\icmlaffiliation{tokyo}{The University of Tokyo, Tokyo, Japan}

\icmlcorrespondingauthor{Ikko Yamane}{ikko.yamane@dauphine.psl.eu}

\icmlkeywords{Weakly supervised learning, semi-supervised learning}

\vskip 0.3in
]



\customPrintAffiliationsAndNotice{}{%
This is a revised version of our ICML paper appearing in \textit{Proceedings of the
$\mathit{38}^{th}$ International Conference on Machine Learning},
 PMLR 139, 2021, with correction to Figure~\ref{fig:toy_dim_vs_mse}.
Copyright 2021 by the author(s).}


\begin{abstract}
  Ordinary supervised learning is useful when we have paired training data of input $X$ and output $Y$.
  However, such paired data can be difficult to collect in practice.
  In this paper, we consider the task of predicting $Y$ from $X$ when we have no paired data of them, but we have two separate, independent datasets of $X$ and $Y$ each observed with some mediating variable $U$,
  that is, we have two datasets $S_X = \{(X_i, U_i)\}$ and $S_Y = \{(U'_j, Y'_j)\}$.
  A naive approach is to predict $U$ from $X$ using $S_X$ and then $Y$ from $U$ using $S_Y$, but we show that this is not statistically consistent.
  Moreover, predicting $U$ can be more difficult than predicting $Y$ in practice, e.g., when $U$ has higher dimensionality.
  To circumvent the difficulty, we propose a new method that avoids predicting $U$ but directly learns $Y = f(X)$ by training $f(X)$ with $S_{X}$ to predict $h(U)$ which is trained with $S_{Y}$ to approximate $Y$.
  We prove statistical consistency and error bounds of our method and experimentally confirm its practical usefulness.
\end{abstract}

\section{Introduction}
Supervised learning methods have been popular as powerful tools for many prediction tasks when we have training data consisting of direct correspondences between the output variable \(Y\) to be predicted and the input variable \(X\) to be used for the prediction~\cite{murphy_machine_2012, mohri2012foundations, shalev-shwartz_understanding_2014}.

However, in some applications, it is difficult or expensive to collect training data consisting of \((X, Y)\)-pairs~\citep{chapelle_semi-supervised_2006, zhu_semi-supervised_nodate, van_engelen_survey_2020}.
For example, consider the case where we want to learn a function predicting sentiment of an image~\citep{mittal_image_sentiment}.
Even if we do not have image data labeled with sentiment information, we might be able to find separate datasets
consisting of images with text captions~\citep{xu_show_2015}
and texts with sentiment labels~\citep{medhat_sentiment_2014}.
Another example is translation between minor languages.
If there are bilingual corpora in those languages, we could apply supervised learning techniques.
However, it can be hard to obtain such training data since there may not be many speakers bilingual in minor languages.
Instead, we may have a better chance to find separate translation corpora in each language with a major one such as English.

In this paper, we consider the situation in which
we do not have access to direct correspondences between \(X\) and \(Y\),
but we only have two separate datasets of \(X\) and \(Y\),
each observed with some \emph{mediating variable} \(U\),
\(S_X = \{(X_i, U_i)\}\) and \(S_Y = \{(U'_j, Y'_j)\}\),
where \((X_i, U_i)\) and \((U'_j, Y'_j)\) are independent and thus we have no paired data of \(X\) and \(Y\).
Note that $U_{i}$ and $U'_{j}$ are generally different samples.
In the example of image sentiment prediction, text captions can be used as the mediating variable \(U\);
\(S_X\) corresponds to image data with text captions,
and \(S_Y\) corresponds to text data with sentiment labels.
We call this framework \emph{mediated uncoupled learning}.

A naive approach is to separately learn the function \(U = g(X)\) using \((X, U)\)-data and the function \(Y = h(U)\) using \((U, Y)\)-data.
Then, one can predict \(Y\) from \(X\) by chaining the estimated functions as
\(\widehat{Y} \coloneqq \widehat{h}(\widehat{U})\) with \(\widehat{U} \coloneqq \widehat{g}(X)\),
where \(\widehat{h}\) and \(\widehat{g}\) are estimates of \(h\) and \(g\), respectively.
However, we show that this method is not statistically consistent
since the point prediction \(\widehat{U}\) does not carry enough information for predicting \(Y\),
unless \(Y\) and \(U\) have linear relationship or \(U\) is a deterministic function of \(X\) (see the detailed discussion in Section~\ref{sec:naive}).

One can fix the inconsistency by (implicitly or explicitly) estimating the conditional probability density function (p.d.f.) \(p(u\given x)\)
of \(U\) given \(X\) in place of the deterministic function \(g\).
Then, one can predict \(Y\) given \(X=x\) by calculating \(\int h(u) \widehat{p}(u \given x) \drm u\)
with the estimated conditional p.d.f.\@ \(\widehat{p}(u \given x)\).
However, this approach involves the task of conditional density estimation or learning generative models, which needs delicate modeling and training~\citep{salimans_improved_gans, gulrajani_improved_wgan, kingma_improved_variational}.
Moreover, it requires integrating the estimated function at the prediction time, which can be computationally inefficient.

A cause of the weaknesses of these naive methods is that they try to predict \(U\),
which is unnecessary in order to solve the original task of predicting \(Y\).
To circumvent this issue, we propose a method that learns a function \(f\) directly predicting \(Y\) from \(X\)
without attempting to predict $U$.
Our proposed method first learns the correspondence $Y = h(U)$, and then train $f$ so that
$f(X)$ will best predict the output of $h(U)$.
This simple approach allows us to use state-of-the-art supervised learning methods as building blocks out of the box
while providing excellent theoretical and practical properties.
Our theoretical analysis shows the statistical consistency and provides an excess error bound for our method.
Finally, we demonstrate the practical usefulness of the proposed method through experiments.

\section{Problem Setup}
\label{sec:problem_setting}
Our goal is to estimate a function that predicts a \(\mathcal{Y}\)-valued output variable \(Y\) from an \(\mathcal{X}\)-valued input variable \(X\),
where \(\mathcal{X}\subseteq \Re^{d_X}\) (\(d_{X} \in \Nat\)) and \(\mathcal{Y} \subseteq \Re\) are measurable spaces,%
\footnote{We can easily extend all the results to the case where \(\mathcal{Y}\) is multi-dimensional.
This manuscript focuses on the one-dimensional case for the ease of notation.}
and \((X, Y)\) follows an unknown probability distribution with density \(p(x, y)\).
More specifically, the function that we want to estimate is
the conditional expectation of \(Y\) given \(X\), which is characterized as the minimizer of the mean squared error (MSE):
\(f^* \coloneqq \E[Y \given X] = \argmin_{f \in L^2_{X}} \E[(f(X) - Y)^2]\),
where \(L^2_{X} \coloneqq \{f: \mathcal{X} \to \Re \given \norm{f}_2^2 \coloneqq \E[f(X)^2] < \infty\}\) and \(\E[\cdot]\) denotes the expectation over all involved variables.
Note that the minimizer is unique in the sense that any minimizer \(f\)
has distance zero from \(f^*\) in the \(L^2_{X}\)-norm: \(\norm{f - f^*}_2^2 = \E[(f(X) - Y)^2] - \E[(f^*(X) - Y)^2] = 0\).
The MSE can be used for classification too, which corresponds to adopting the squared loss as a surrogate loss.%
\footnote{In the case of multi-class classification, we can use the squared loss with the one-hot representation for class labels.}

Unlike standard supervised problems, we have no access to direct supervision provided by joint samples of \((X, Y)\).
Instead, we assume that there exists a \(\mathcal{U}\)-valued \emph{mediating variable} \(U\) for which there is a joint density function \(p(x, u, y)\) of \((X, U, Y)\), where \(\mathcal{U}\subseteq \Re^{d_U}\) (\(d_{u} \in \Nat\)) is a measurable space,
and we are given two sets of i.i.d.\@ samples, \(\{(X_i, U_i)\}_{i=1}^n \iid p(x, u)\) and \(\{(U'_i, Y'_i)\}_{i=1}^{n'} \iid p(u, y)\).
Here, \(p(x, u)\) and \(p(u, y)\) are the marginal p.d.f.-s of \((X, U)\) and \((U, Y)\), respectively, that are compatible with \(p(x, u, y)\).
We call these data \emph{mediated uncoupled data} since \(X\) and \(Y\) are observed separately with a common variable \(U\) that mediates between them.

We assume the \emph{conditional mean independence} given by
\begin{equation}
    \E[Y \given U] = \E[Y \given U, X = x] \label{eq:u_sufficient}
\end{equation}
for every \(x \in \cX\).
The condition ensures that \((U, Y)\) has enough information to learn \(\E[Y \given X]\).
Note that the conditional independence of \(Y\) and \(X\) given \(U\), i.e., \(Y \indep X \given U\),
implies the conditional mean independence, Eq.~\eqref{eq:u_sufficient}, but the converse is not true.
We discuss the case where this assumption is not satisfied in Section~\ref{sec:assump}.

A related but different problem setting of learning without input-output correspondences was studied in \citet{zhang2019learning}.
They considered the situation in which, using our notation, \(U \indep X \given Y\) and the conditional probability of \(Y\) given \(U\) is given instead of \((U, Y)\)-pairs, which is a typical scenario in learning with noisy labels~\citep{angluin_learning_1988, blanchard_classification_2016, natarajan_learning_2013}.
Also, these methods focus on the case where \(Y\) is discrete.
\citet{yamane_uplift_2018} considered estimation of causal effect using data separately labeled with either treatment or outcome
but not both at the same time.
The type of training data is similar to ours, but the problem setup is essentially different.

\citet{dhir_integrating_2020} proposed a method for causal discovery under a similar setup in which not all combinations of the variables of interest are jointly observed. The form of data that they assumed includes ours as a special case, and they provided some real-world examples in which such data arise. However, their goal is to infer causal directions betweeen variables but not to predict their values, and thus their method is not applicable to our problem.

\section{Naive Approach Based on Separate Estimators}
\label{sec:naive}
A naive approach to this problem is to estimate \(g^*(x) \coloneqq \E[U \given X = x]\) and \(h^*(u) \coloneqq \E[Y \given U = u]\) as \(g(x)\) and \(h(u)\), respectively, and then combine them as \(f_\text{combine}(x) \coloneqq h(g(x))\) to estimate \(\E[Y \given X = x]\).
The combined estimator is consistent when \(h^*\) is a linear function, or \(U\) is a deterministic function of \(X\) so that
\begin{equation*}
    h^*(g^*(x))
    = h^*(\E[U \given X = x])
    = \E[h^*(U) \given X = x],
\end{equation*}
in which case, we have
\begin{align}
    h^*(g^*(x))
    &= \E[h^*(U) \given X = x]\nonumber\\
    &= \E[\E[Y \given U] \given X = x]\nonumber\\
    &= \E[\E[Y \given U, X = x] \given X = x] \quad \text{(from Eq.~\eqref{eq:u_sufficient})}\nonumber\\
    &= \E[Y \given X = x]. \label{eq:hstar_given_x}
\end{align}
Hence, the consistency of estimators of \(h^*\) and \(g^*\) will guarantee the consistency of their composite function to \(\E[Y\given X]\).

However, it fails to consistently estimate the target function \(\E[Y \given X = x]\) in many important non-linear cases.
For example, when \(h^*\) is strictly convex, and \(U\) is a stochastic function of \(X\),
Jensen's inequality implies
\begin{align*}
    h^*(g^*(x))
    &= h^*(\E[U \given X = x])\\
    &< \E[h^*(U) \given X = x]\\
    &= \E[Y \given X = x],
\end{align*}
where the last line follows from Eq.~\eqref{eq:hstar_given_x}.
Thus, the estimator under-estimates the target, and it is not consistent.

One can develop a consistent version of the naive method
by estimating the conditional density function \(p(u \given x)\) of \(U\) given \(X\) instead of the conditional expectation \(\E[U \given X = x]\).
Once \(p(u \given x)\) is estimated as \(q(u \given x)\) and \(\E[Y \given U = u]\) as \(h(u)\),
one can estimate \(\E[Y \given X = x]\) as
\begin{align}
    f_\text{integral}(x) \coloneqq \int h(u) q(u \given x) \mathrm{d}u. \label{eq:integral_approach}
\end{align}
Then, \(f_\text{integral}\) is a consistent estimator of \(\E[Y\given X = x]\) as long as \(q(u \given x)\) and \(h(u)\) are consistent because it converges to
\begin{align*}
    &\int h^*(u) p(u \given x) \mathrm{d}u\\
    &= \int \E[Y \given U = u] p(u \given x) \mathrm{d}u\\
    &= \int \E[Y \given U = u, X = x] p(u \given x) \mathrm{d}u \quad \text{(by Eq.~\eqref{eq:u_sufficient})}\nonumber\\
    &= \E[Y \given X = x].
\end{align*}
However, this modified method solves the hard intermediate problem of estimating the conditional probability density function \(p(u \given x)\).
This can be particularly problematic when we use neural networks because it has been reported that they tend to be overconfident and show poor performance in predicting the conditional probability of the output~\citep{hein_overconf_2019}.
Moreover, one needs to be able to accurately calculate the integral in Eq.~\eqref{eq:integral_approach}, e.g., by sampling from \(q(u \given x)\), at each prediction,
which is often computationally demanding and prohibitive when we need real-time responses in prediction.
Although the efficient belief-propagation based algorithm proposed by \citet{pmlr-v9-song10a} can be used when distributions are represented by reproducing kernel Hilbert space (RKHS) embeddings, it is not generally applicable when we use function classes other than RKHSs, such as neural networks.

In fact, for several problems that are solvable by performing density estimation as intermediate tasks,
directly solving the target task without solving density estimation
reportedly improves performance~\cite{sugiyama_suzuki_kanamori_2012, sugiyama_density-difference_2013, Sasaki14}.
\citet{vapnik_nature_stat_learn_1995} also argued that it is preferable to avoid solving intermediate tasks
that are more general than the target task.

Another higher-level criticism of the naive approach above from a statistical point of view is that
the intermediate step of estimating \(\E[U \given X]\) (or \(p(u \given x)\)) is performed without any attention to the target task of predicting \(Y\).
This means that those estimators are not designed in a way that the resulting prediction for \(Y\) will be accurate.

\section{Proposed Methods}
Our approach learns a function \(f\colon \mathcal{X} \to \mathcal{Y}\) that directly predicts \(Y\) from \(X\)
without predicting \(U\).
It does not try to solve the hard intermediate problem of predicting \(U\) from \(X\),
and it is free from the issues that the naive approach suffers.
Below, we describe our approach to the problem more precisely and propose two methods based on it.

\subsection{Two-step Regressed Regression (2Step-RR)}
The first proposed method consists of two steps.
The first step is for training a function \(h: \mathcal{U} \to \mathcal{Y}\) to predict \(Y\) from \(U\) using \(S_{Y}\).
Because each sample of \(U\) in \(S_{{Y}}\) is labeled with the corresponding sample of \(Y\),
this step is no more than an ordinary supervised learning task.
Now, \(h\) can predict \(Y\), but its input is \(U\), not \(X\).
To obtain a function \(f\) that takes \(X\) as input and predicts \(Y\),
we train \(f\) so that the output of \(f(X)\) will mimic that of \(h(U)\),
for which we only need \(S_{X}\) consisting of samples of \((X, U)\).

More specifically, we first train a function \(\widetilde{h} \colon \mathcal{U} \to \mathcal{Y}\) for predicting \(Y\) from \(U\):
\begin{equation*}
    \widetilde{h} = \argmin_{h\in\mathcal{H}} \frac{1}{n'}\sum_{i=1}^{n'}[(h(U'_i) - Y'_i)^2].
\end{equation*}
Then, we train another function \(\widetilde{f} \colon \mathcal{X} \to \mathcal{Y}\) for predicting \(\widetilde{h}(U)\) from \(X\):
\begin{equation*}
    \widetilde{f} = \argmin_{f\in\mathcal{F}} \frac{1}{n}\sum_{i=1}^n[(f(X_i) - \widetilde{h}(U_i))^2].
\end{equation*}
In the above, \(\mathcal{H}\) and \(\mathcal{F}\) are function classes for \(\widetilde{h}\) and \(\widetilde{f}\), respectively.
Because we train \(\widetilde{h}\) so that \(\widetilde{h}(U)\) will predict \(Y\) well,
\(\widetilde{f}(X)\) is expected to predict \(Y\) well by predicting \(\widetilde{h}(U)\).
We call this method \emph{Two-Step Regressed Regression (2Step-RR)} since the predictive function \(f\) is learned by regressing the regression of \(Y\).

Note that the objective functional in each step uses samples of either \((X, U)\) or \((U, Y)\),
not both at the same time.
Thus, we can compute it with our mediated uncoupled data,
\(\{(X_i, U_i)\}_{i=1}^n\) and \(\{(U'_i, Y'_i)\}_{i=1}^{n'}\).
We summarize the algorithm in Algorithm~\ref{alg:twostep}.
In Section~\ref{sec:theory}, we will show that this method has statistical consistency
and admits a nice non-asymptotic error bound.

\begin{algorithm}[tb]
   \caption{Two-Step Regressed Regression (2Step-RR)}
   \label{alg:twostep}
\begin{algorithmic}
   \STATE \(\widetilde{h} \gets \argmin_{h\in\mathcal{H}} \frac{1}{n'}\sum_{i=1}^{n'} (h(U'_i) - Y'_i)^2.\)
   \STATE \(\widetilde{f} \gets \argmin_{f\in\mathcal{F}} \frac{1}{n}\sum_{i=1}^n (f(X_i) - \widetilde{h}(U_i))^2.\)
   \STATE {\bfseries Return:} \(\widetilde{f}\)
\end{algorithmic}
\end{algorithm}

\subsection{Jointly Performing the Two Steps}
2Step-RR described above has nice theoretical properties (see Section~\ref{sec:theory}),
but there may be room for improvement on practical, finite-sample performance.
More specifically, while 2Step-RR uses no information for training \(f\) when \(h\) is trained,
it may be advantageous to let \(h\) adapt to the second step
in a way that it will be easier for \(f(X)\) to fit \(h(U)\).
Here, we are going to combine the two steps of 2Step-RR
to develop a variant called \emph{Joint Regressed Regression (Joint-RR)} that trains \(f\) and \(h\) at the same time.
This allows \(h\) to incorporate how well \(f(X)\) can fit \(h(U)\)
and adjust itself in favor of the training of \(f\).

\subsubsection{Joint Regressed Regression (Joint-RR)}
The procedure of Joint-RR itself is simple (see Algorithm~\ref{alg:onestep}).
We additively combine the two objective functionals used for training \(f\) and \(h\) in 2Step-RR:
\begin{align}
  \widehat{J}_w(f, h)
  &\coloneqq \frac{1}{wn}\sum_{i=1}^n (f(X_i) - h(U_i))^2\nonumber\\
  &\phantom{\coloneqq\ }+ \frac{1}{(1-w)n'}\sum_{i=1}^{n'} (h(U'_i) - Y'_i)^2, \label{eq:ub_sample}
\end{align}
where \(w \in (0, 1)\) is a weight parameter.
Then, we minimize Eq.~\eqref{eq:ub_sample} with respect to \(f\) and \(h\) jointly:
\begin{equation*}
    (\widehat{f}_{w}, \widehat{h}_{w}) \coloneqq \argmin_{f \in \mathcal{F}, h\in \mathcal{H}} \widehat{J}_w(f, h).
\end{equation*}

\begin{algorithm}[tb]
   \caption{Joint Regressed Regression (Joint-RR)}
   \label{alg:onestep}
\begin{algorithmic}
   \STATE \((\widehat{f}_{w}, \widehat{h}_{w}) \coloneqq \argmin_{f \in \mathcal{F}, h\in \mathcal{H}} \widehat{J}_w(f, h)\)
   (see Eq.~\eqref{eq:ub_sample}.)
   \STATE {\bfseries Return:} \(\widehat{f}_{w}\)
\end{algorithmic}
\end{algorithm}

In the rest of this section, we will give more detailed justification of Joint-RR as upper bound minimization.

\subsubsection{Joint-RR as Upper Bound Minimization}
\label{sec:ub_min}
We start by constructing an upper bound of the population version of the MSE that can be approximated with our mediated uncoupled data, \(S_{X}\) and \(S_{Y}\).
\begin{theorem}\label{thm:ub}
  The MSE can be bounded as
  \begin{equation}
    \E[(f(X) - Y)^2] \le J_w(f, h), \label{eq:ub_any_h}
  \end{equation}
  where
  \begin{align*}
    J_w(f, h)
    &\coloneqq \frac{1}{w}\E[(f(X) - h(U))^2]\\
    &\phantom{\coloneqq\ } + \frac{1}{1-w}\E[(h(U) - Y)^2]
  \end{align*}
  for any \(w \in (0, 1)\) and any \(h \in L_U^2 \coloneqq \{h\colon\mathcal{U} \to \mathcal{Y} \given \E[h(U)^2] < \infty\}\).
\end{theorem}

A proof is in Appendix~\ref{apx:proof_ub} in the supplementary material.
The functional \(J_w(f, h)\) has terms each involving either \((X, U)\) or \((U, Y)\) but not both at the same time.
This convenient property allows us to approximate it with the mediated uncoupled data, \(S_{X} = \{(X_i, U_i)\}_{i=1}^n\) and \(S_{Y} = \{(U'_i, Y'_i)\}_{i=1}^{n'}\).
Among the upper bounds of the form in Eq.~\eqref{eq:ub_any_h}, we take the tightest one with respect to \(h\):
\begin{equation}
    \E[(f(X) - Y)^2] \le \min_{h\in L_X^2} J_w(f, h). \label{eq:min_is_still_ub}
\end{equation}
The minimization of the right-hand side of Eq.~\eqref{eq:min_is_still_ub} yields the population version of our optimization problem:
\begin{equation}
    \argmin_{f \in L_X^2} \min_{h\in L_U^2} J_w(f, h). \label{eq:ubmin_pop}
\end{equation}
In practice, we solve its empirical version with some hypothesis classes \(\mathcal{F}\) and \(\mathcal{H}\) for \(f\) and \(h\), respectively, to obtain our estimator:
\begin{equation}
    \widehat{f}_{w} \coloneqq \argmin_{f \in \mathcal{F}} \min_{h\in \mathcal{H}} \widehat{J}_w(f, h),  \label{eq:ubmin_sample}
\end{equation}
where \(\widehat{J}_w(f, h)\) is defined by Eq.~\eqref{eq:ub_sample}.

The argument above is valid for any \(w \in (0, 1)\).
How to optimize \(w\) so as to minimize the test MSE
is not trivial since we do not have \((X, Y)\)-data for validation,
and we leave this as an open question.
In our experiments, we simply fixed \(w\) to the balanced value \(1/2\), which performs well in many cases
and tends to show more stable performance compared to 2Step-RR.
In Section~\ref{sec:conn_one_two}, we will show that 2Step-RR is the limit of Joint-RR with \(w\to 1\).

\subsubsection{Closed-form Solution for Linear-in-parameter Models}
When interpretation or fast prediction is required, linear-in-parameter-models would be useful.
Fortunately, our methods admit closed-form solutions for those models.
Due to the limited space, we only present the solution to Joint-RR here.
The derivation for 2Step-RR is more straightforward.

\begin{theorem}
  Let \(f_{\bm\alpha}(\bm x) \coloneqq \bm\alpha^\top \bm\phi(\bm x)\),
  \(h_{\bm\beta}(\bm u) \coloneqq \bm\beta^\top \bm\psi(\bm u)\),
  \(\bm\theta \coloneqq (\bm\alpha^\top, \bm\beta^\top)^\top\),
  and \(\lambda \in (0, \infty)\).
  Then, the \(\ell_{2}\)-regularized solution
  \begin{equation*}
    (\widehat{\bm\alpha}, \widehat{\bm\beta})
    \coloneqq \argmin_{({\bm\alpha}, {\bm\beta})\in\Re^{b_{\cF}} \times \Re^{ b_{\cH}}} \left[ \widehat{J}_w(f_{\bm\alpha}, h_{\bm\beta}) + \lambda \bm\theta^\top\bm\theta \right]
  \end{equation*}
  is given by
  \begin{align}
    \widehat{\bm\alpha}
      &\coloneqq \bm M_1^{-1}\bm M_2 \widehat{\bm\beta},\quad\text{and}\label{eq:fast_alpha}\\
    \widehat{\bm\beta}
      &\coloneqq (\bm M_3 - \bm M_2^\top\bm M_1^{-1}\bm M_2)^{-1} \bm b_1,  \label{eq:fast_beta}
  \end{align}
  where
  \begin{align*}
    \bm M_1 &\coloneqq \frac{1}{nw}\sum_{i=1}^n\bm\varphi(X_i)\bm\varphi(X_i)^\top + \lambda\bm I_{b_{\cF}},\\
    \bm M_2 &\coloneqq \frac{1}{nw}\sum_{i=1}^n\bm\varphi(X_i)\bm\psi(U_i)^\top,\\
    \bm M_3 &\coloneqq \frac{1}{nw}\sum_{i=1}^n\bm\psi(U_i)\bm\psi(U_i)^\top\\
            &+ \frac{1}{n'(1-w)}\sum_{i=1}^{n'}\bm\psi(U'_i)\bm\psi(U'_i)^\top + \lambda\bm I_{b_{\cH}},\\
    \bm b_1 &\coloneqq \frac{1}{n'(1-w)}\sum_{i=1}^{n'}Y'_i\bm\psi(U'_i).
  \end{align*}
\end{theorem}

Eqs.~\eqref{eq:fast_alpha} and \eqref{eq:fast_beta} involve matrices of size at most \(\max(b_{\cF}, b_{\cH})\)-by-\(\max(b_{\cF}, b_{\cH})\),
which requires less computational resources in terms of both space and time compared to
a naive solution involving the inversion of a \((b_{\cF} + b_{\cH})\)-by-\((b_{\cF} + b_{\cH})\) matrix.
Details and a proof are presented in Appendix~\ref{apx:linearmodels} in the supplementary material.

\section{Theoretical Analysis}\label{sec:theory}
In this section, we present several theoretical results.

\subsection{Connection between 2Step-RR and Joint-RR}
\label{sec:conn_one_two}
2Step-RR and Joint-RR have an interesting connection that helps us better understand them.
Briefly speaking, 2Step-RR can be seen as a special case of Joint-RR
in the sense that we can obtain the former by taking the limit of the latter with \(w \to 1\).
The following theorem provides a more formal statement on the connection.
\begin{theorem}\label{thm:proactive-two-step}
    Suppose that \(\cF \subseteq L^2_X\) and \(\cH \subseteq L^2_U\)
    satisfy
    \(h^*(u) \coloneqq \E[Y \given U = u] \in \cH\)
    and \(f^*(x) \coloneqq \E[h^*(U) \given X = x] \in \cF\).
    Then,
    \begin{equation}
        (f^*, h^*) \in \lim_{w \toup 1} \argmin_{(f, h) \in \cF \times \cH}
        J_w(f, h) \label{eq:proposed_minus_C}.
    \end{equation}
\end{theorem}

We present a proof in Appendix~\ref{sec:proof_two_step} in the supplementary material.
Note that \((f^*, h^*)\) is the solution pair to the optimization problem of 2Step-RR with the population-level objective functionals:
\begin{align*}
    h^* &\in \argmin_{h \in \cH} \E[(h(U) - Y)^2]\\
    \quad \text{and} \quad
    f^* &\in \argmin_{f \in \cF} \E[(f(X) - h^*(U))^2].
\end{align*}
The theorem states that \((f^*, h^*)\) is also equal to the limit of the solution pair to the Joint-RR optimization problem with the population-level objective functional in Eq.~\eqref{eq:ubmin_pop}.

\subsection{Statistical Consistency of 2Step-RR}
Let us informally confirm the statistical consistency of 2Step-RR
under the conditional mean independence (Eq.~\eqref{eq:u_sufficient}).
When the models are correctly specified, and Eq.~\eqref{eq:u_sufficient} and appropriate convergence conditions hold,
the 2Step-RR estimator \(\widetilde{f}\) converges as \(n \to \infty\) and \(n' \to \infty\) to
\begin{align*}
    f^*(x)
    &= \E[h^*(U) \given X = x]\\
    &= \E[\E[Y \given U] \given X = x]\\
    &= \E[\E[Y \given U, X] \given X = x]\quad\text{(from Eq.~\eqref{eq:u_sufficient})}\\
    &= \E[Y \given X = x].
\end{align*}
Thus, it is consistent under the condition of Eq.~\eqref{eq:u_sufficient}.
We can formally confirm this as a corollary of Theorem~\ref{thm:gen_bound_pts} given later.

\subsection{Joint-RR is a Regularized Method}
\label{sec:UB_is_regularized}
Unlike 2Step-RR, Joint-RR is not statistically consistent,
but it can be seen as a nice regularized counterpart
in the sense that its objective function is the MSE plus
the deviation of \(f(X)\) from the conditional mean \(\E[f(X) \given U]\).

Under the assumption in Eq.~\eqref{eq:u_sufficient}, we have
\begin{align*}
    &\MSE(f) \coloneqq \E[(Y - f(X))^2]\\
    &= \E[(\E[Y \given U] - f(X))^2]
      + \E[(Y - \E[Y \given U])^2],
\end{align*}
where the last term is a constant that does not depend on \(f\).
Hence,
\begin{align}
    &\min_{h\in L_U^2} J_w(f, h)\nonumber\\
    &= \E[(\E[Y \given U] - f(X))^2]\nonumber\\
    &\phantom{=\ } + \frac{1-w}{w} \E[(f(X) - \E[f(X) \given U])^2]
    +\ \text{const.}\nonumber\\
    &= \MSE(f) + \text{const.}\nonumber\\
    &\phantom{=\ } + \underbrace{\frac{1-w}{w} \E[(f(X) - \E[f(X) \given U])^2]}_{\text{The shrinkage regularizer.}}
    \label{eq:gap_analysis}
\end{align}
for any \(w \in (0, 1)\).
See Appendix~\ref{section:calc_UB_reg} in the supplementary material for a more detailed calculation.

This shows that Joint-RR with \(w < 1\) minimizes a biased objective functional.
However, it is often favorable to trade off some bias for smaller variance in practice.
We can also see that the amplitude of the shrinkage term can be controlled by \(w\)
and it vanishes at the limit of \(w \to 1\).
In our experiments, we simply fixed \(w\) to the balanced value \(1/2\), which performs well in many cases.

\subsection{Excess Error Bound}
2Step-RR solves a simple least-squares problem twice.
This allows us to derive a non-asymptotic bound of the excess error
\(\MSE(\widetilde{f}) - \MSE(f^{\dagger})\), where \(f^{\dagger}(x) \coloneqq \E[Y \given X = x]\),
in terms of the Rademacher complexities of function classes.
\begin{theorem}\label{thm:gen_bound_pts}
  Let \(C_{\cF} \coloneqq \sup_{{f\in\cF, x\in\cX}}f(x) < \infty\),
  \(C_{\cH} \coloneqq \sup_{{h\in\cH, u\in\cU}}f(u) < \infty\),
  and \(C_{\cY} \coloneqq \sup \cY < \infty\).
  Let \(\mathfrak{R}_{n}(\cF)\) denote the Rademacher complexity of \(\cF\)
  over \(\{(X_{i}, U_{i})\}_{i=1}^{n}\)
  and \(\mathfrak{R}_{n'}(\cH)\) denote that of \(\cH\)
  over \(\{(U'_{i}, Y'_{i})\}_{i=1}^{n'}\) (see the exact definitions in Appendix~\ref{sec:rademacher} in the supplementary material).
  Suppose that \(h^{*} \in \cH\), \(f^{*} \in \cF\), and \(f^{\dagger}\in\cF\).
  Then, the excess error can be bounded with probability at least \(1 - \delta\) as
  \begin{align*}
    &\E[(\widetilde{f}(X) - f^{\dagger}(X))^{2}] = \MSE(\widetilde{f}) - \MSE(f^{\dagger})\\
    &\le 8 (C_{\cF} + C_{\cH}) (\mathfrak{R}_{n}(\cF) + \mathfrak{R}_{n}(\cH))\\
    &\phantom{\le\ } + 8 (C_{\mathcal{H}} + C_{\mathcal{Y}}) \mathfrak{R}_{n'}(\mathcal{H})\\
    &\phantom{\le\ } + 4 (C_{\cF} + C_{\cH})^2\sqrt{\frac{2}{n} \log\frac{1}{\delta}}\\
    &\phantom{\le\ } + 2 (C_{\cH} + C_{\cY})^2\sqrt{\frac{2}{n'}\log\frac{1}{\delta}}\\
    &\le \mathcal{O}_{p}\bigg(
      \mathfrak{R}_{n}(\cF)
      + \mathfrak{R}_{n}(\cH)
      + \mathfrak{R}_{n'}(\cH)
      + \frac{1}{\sqrt{n}} + \frac{1}{\sqrt{n'}}\bigg).
  \end{align*}
\end{theorem}
For instance,
when \(\cF\) and \(\cH\) are bounded linear-in-parameter models,
\(\mathfrak{R}_{n}(\cF) = \mathcal{O}(1/\sqrt{n})\),
\(\mathfrak{R}_{n}(\cH) = \mathcal{O}(1/\sqrt{n})\),
and \(\mathfrak{R}_{n'}(\cH) = \mathcal{O}(1/\sqrt{n'})\)~\citep{mohri2012foundations}.
Using Theorem~\ref{thm:gen_bound_pts}, we can bound the excess error
by \(\mathcal{O}_{p}(1/\sqrt{n} + 1/\sqrt{n'})\).
A proof is in Appendix~\ref{sec:gen_bound_pts} in the supplementary material.

\subsection{Discussion on the Assumption}
\label{sec:assump}
So far, we have focused on the ideal case in which the conditional mean independence (Eq.~\eqref{eq:u_sufficient}) holds.
However, it may be difficult to exactly ensure the condition in practice.

Here, we relax Eq.~\eqref{eq:u_sufficient} by allowing the gap between the left-hand and right-hand sides to be potentially larger than zero
but bounded by \(c^{2}\) for some constant \(c \in (0, \infty)\).
We will show that (i) even the best possible method suffers an MSE of at least \(c^{2} / 2\)
in the worse-case within this scenario
while (ii) 2Step-RR suffers an MSE of at most \(c^{2} + o(1)\).

To see the claim (ii), notice that we have already shown that \(\widetilde{f}\) converges to \(\E[\E[Y \given U] \given X = (\cdot)]\).
This implies that 2Step-RR suffers an MSE of
\begin{align*}
  &\E[(\E[\E[Y \given U] \given X] - \E[Y \given X])^{2}] + o(1)\\
  &= \E[(\E[\E[Y \given U] - \E[Y \given U, X] \given X])^{2}] + o(1)\\
  &\le \E[(\E[Y \given U] - \E[Y \given U, X])^{2}] + o(1)\\
  &\le c^{2} + o(1)
\end{align*}
under the relaxed assumption.

The following proposition is a formal statement of the claim (i). A proof is in Appendix~\ref{apx:assump} in the supplementary material.
\begin{proposition}
  \label{prop:mmlb_det_informal}
  For any estimator \(\widehat{f}_{(\cdot)}\) that takes mediated uncoupled data \(S_{X}\) and \(S_{Y}\) as input
  and produces a function from \(\cX\) to \(\cY\), we have
  \begin{equation*}
    \sup_{p^* \in \mathcal{P}_{c}}
    \E\lr[]{\lr(){\widehat{f}_{S_{X}, S_{Y}}(X) - \E[Y \given X]}^2}
    \ge \frac{1}{2}c^2,
  \end{equation*}
  where \((X, Y) \sim p^*(x, y)\),
  and \(\mathcal{P}_{c}\) is the class of p.d.f.-s \(\widetilde{p}(x, u, y)\) for which \(\E[(\E[\widetilde{Y} \given \widetilde{X}] - \E[\widetilde{Y} \given \widetilde{X}, \widetilde{U}])^{2}] \le c^{2}\),
  \(\widetilde{p}(x) = \widetilde{p}(-x)\),
  and \(\E[\widetilde{U}] = 0\) with \((\widetilde{X}, \widetilde{U}, \widetilde{Y}) \sim \widetilde{p}(x, u, y)\).
\end{proposition}

Our bound relies on Le Cam's method, which yields the \(1/2\) factor
(see Appendix~\ref{apx:assump} in the supplementary material for details).
Whether one can eliminate the factor remains as future work.

\section{Experiments}
In this section, we present experimental results.

\subsection{Experiments with Synthetic Data}
First, we present experiments with synthetic data.
Because neural networks are becoming the gold standard in many tasks,
we test the methods using neural networks as follows.
\begin{itemize}
  \item The naive method using multi-layer perceptrons
        with four layers, \(20\) hidden units in each layer, and ReLU activations.
        We refer to this method as ``Naive''.
  \item 2Step-RR using multi-layer perceptrons
        with four layers, \(20\) hidden units in each layer, and ReLU activations.
  \item Joint-RR using multi-layer perceptrons
        with four layers, \(20\) hidden units in each layer, and ReLU activations.
        We set \(w = 1/2\).
\end{itemize}
We train all models with Adam~\citep{kingma_adam_2017} for \(200\) epochs.
We implemented the methods using PyTorch~\citep{pytorch_NEURIPS2019}\footnote{The code will be available on \url{https://github.com/i-yamane/mediated_uncoupled_learning}.}
We use the default values of the implementation provided by PyTorch~\citep{pytorch_NEURIPS2019} for all the parameters of Adam: the learning rate is \(0.001\), and \(\beta\) is \((0.9, 0.999)\).
We prepare mediated uncoupled data (defined in Section~\ref{sec:problem_setting}) for training
but ordinary coupled \((X,Y)\)-data for test evaluation.
The task here is regression, and we use the MSE as the evaluation metric.

For our synthetic data, we can surely confirm whether the conditional mean independence of Eq.~\eqref{eq:u_sufficient} holds or not.
We will test the methods with varying dimensionality in both cases in which Eq.~\eqref{eq:u_sufficient} is satisfied and violated.
For the setting satisfying the condition,
we define the data distribution as follows.
\(X\) is distributed uniformly over \([-1, 1]^d\).
\(U_{j} \coloneqq X_{j}^3 + \varepsilon_u\),
where \(U_{j}\) is the \(j\)-th element of \(U\), \(X_{j}\) is the \(j\)-th element of \(X\),
and \(\varepsilon_u\) is a uniform noise over \([-0.5, 0.5]^d\).
\(Y \coloneqq \norm{U}^2 + \varepsilon_y\), where \(\varepsilon_y\) is a Gaussian noise with mean zero and variance \(0.1\).
This satisfies the condition because \(Y\) and \(X\) are independent after conditioning on \(U\).
For the setting violating the condition,
\(X\) and \(U\) are the same as in the case with the condition satisfied,
but \(Y\) depends on \(X\) rather than \(U\):
\(Y \coloneqq \norm{X}^2 + \varepsilon_y\), where \(\varepsilon_y\) is a Gaussian noise with mean zero and variance \(0.1\).
This violates the condition because \(U\) lacks some information that \(X\) has in predicting \(Y\) due to the noise \(\varepsilon_u\).
In both cases, we use \(1{,}000 \times 2\) mediated uncoupled data for training and \(10{,}000\) coupled \((X, Y)\)-data for test evaluation.

\begin{figure}[tp]
\centering
\subcaptionbox{The setting satisfying the conditional mean independence (Eq.~\eqref{eq:u_sufficient}). \label{fig:toy_sat}}{\includegraphics[width=\columnwidth*22/48]{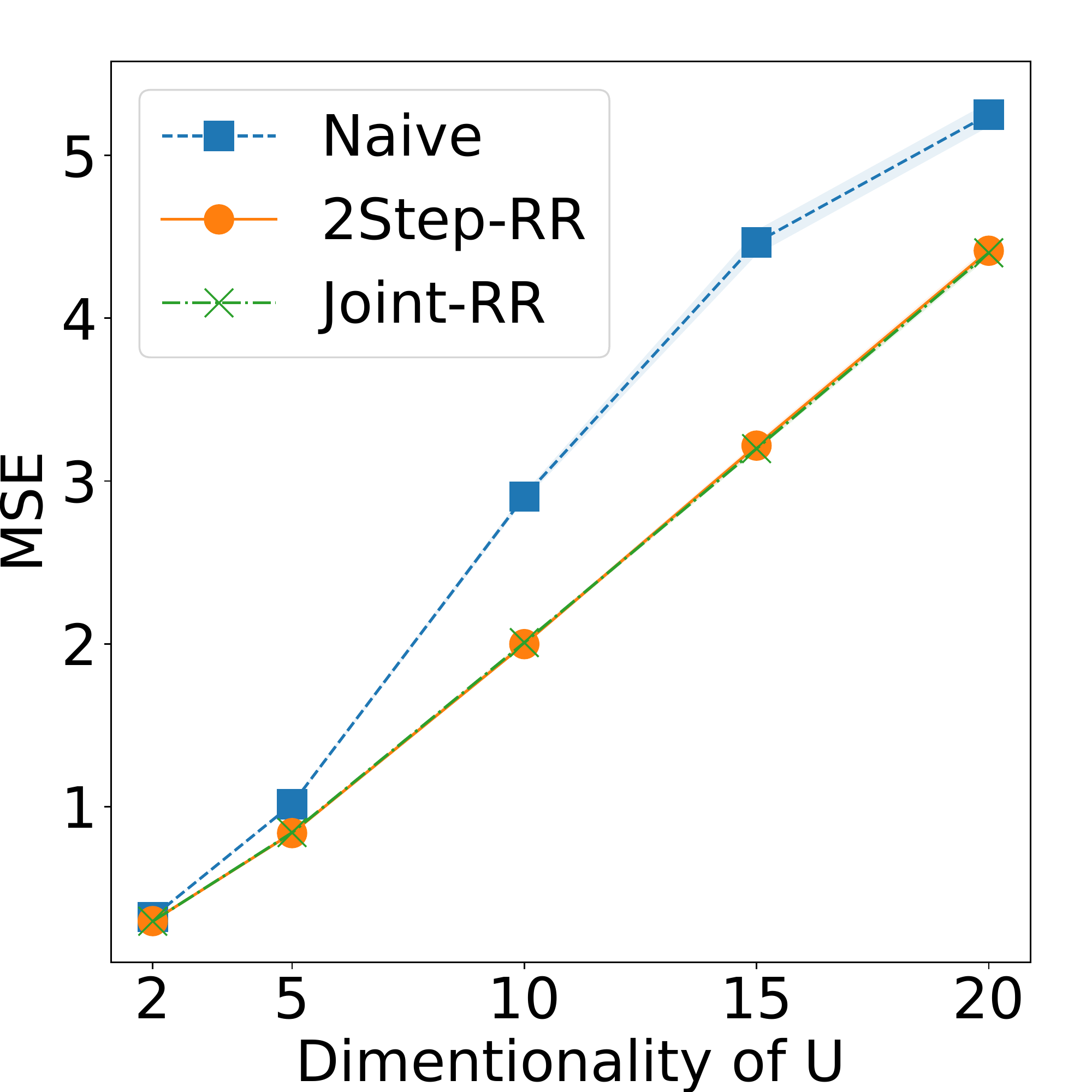}}
\hspace*{2pt}
\subcaptionbox{The setting violating the conditional mean independence (Eq.~\eqref{eq:u_sufficient}). \label{fig:toy_vio}}{\includegraphics[width=\columnwidth*22/48]{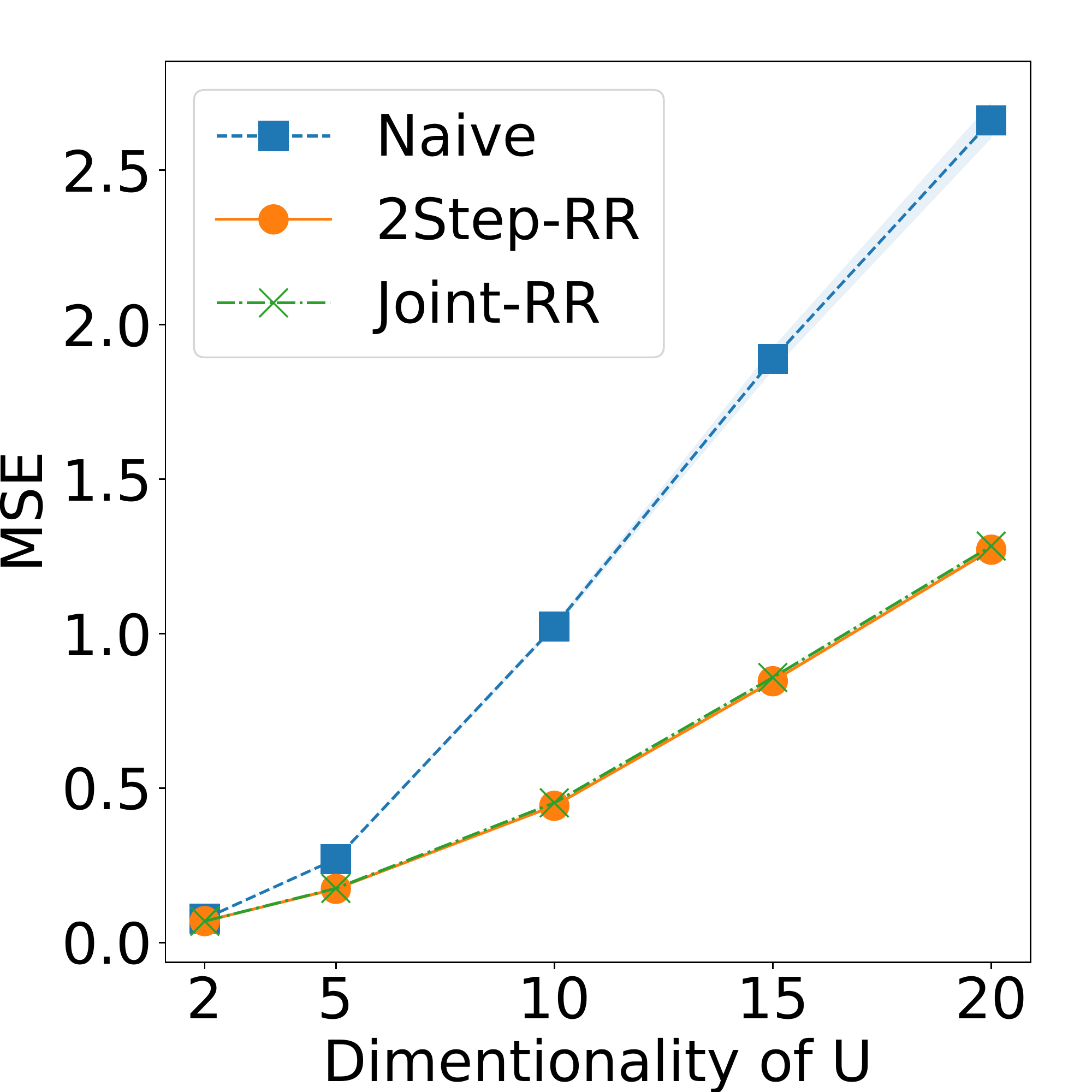}}
\caption{Results for the synthetic data experiments. The plots show the average MSEs and the shaded areas show the standard errors. Note that the standard errors are so small that the shaded area are almost unnoticable.}
\label{fig:toy_dim_vs_mse}
\end{figure}

Results are summarized in Figure~\ref{fig:toy_dim_vs_mse};
Figure~\ref{fig:toy_sat} and \ref{fig:toy_vio} are for the settings satisfying and violating Eq.~\eqref{eq:u_sufficient}, respectively.
The plots show that the proposed methods outperform the naive method.
2Step-RR and Joint-RR gave similar performances and their plots are almost indistinguishable in the figures.

Figure~\ref{fig:synthetic_satisfied} shows more detailed results for each configuration of data dimensionality,
showing that the proposed methods gave consistently lower MSEs than the naive method.
Notably, Joint-RR tends to be more stable than 2Step-RR in the sense that the deviation of the MSE is smaller.

\begin{figure}[tp]
\centering
\subcaptionbox{Dim. 2 \label{fig:syndim2}}{\includegraphics[width=\columnwidth*23/48]{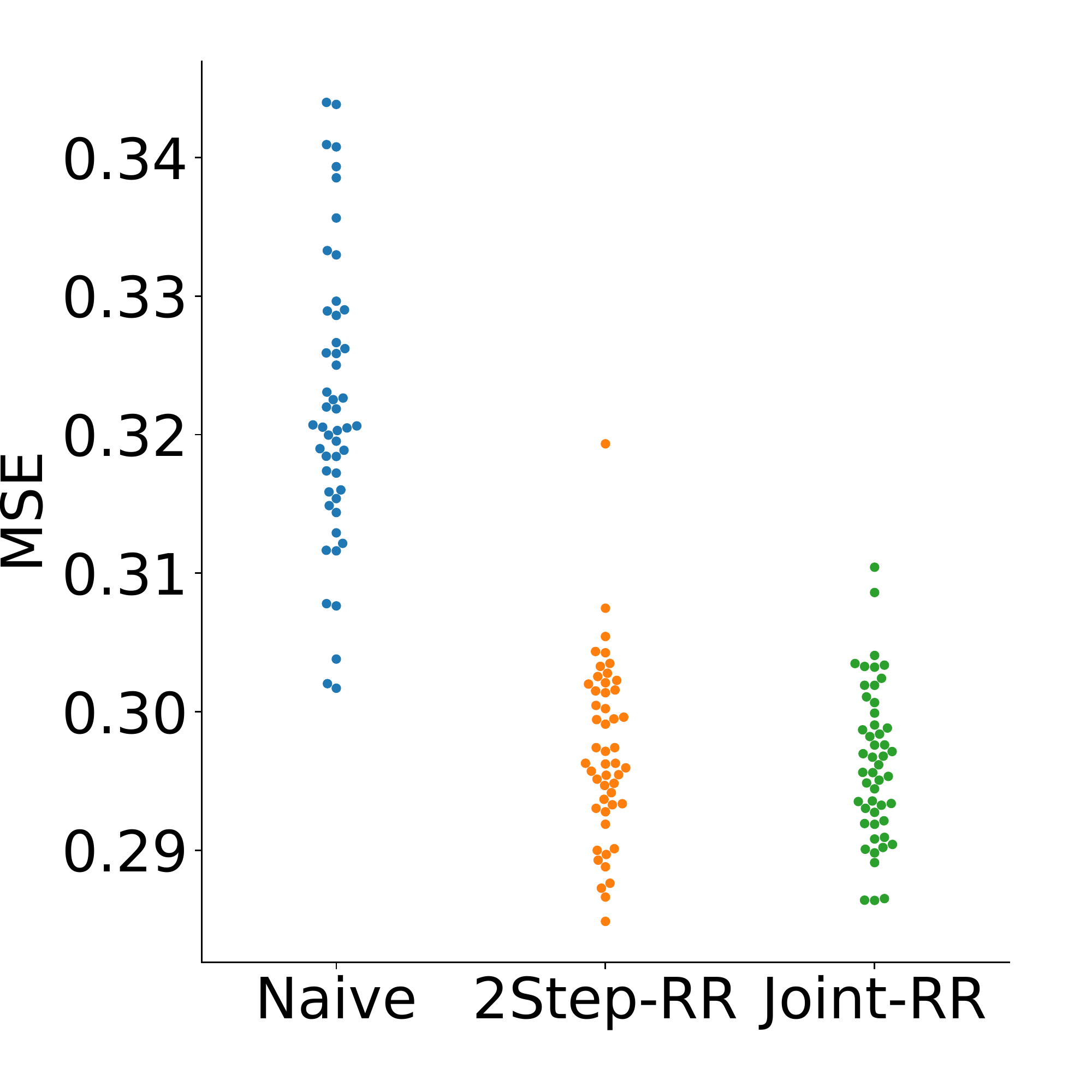}}
\subcaptionbox{Dim. 5 \label{fig:syndim2}}{\includegraphics[width=\columnwidth*23/48]{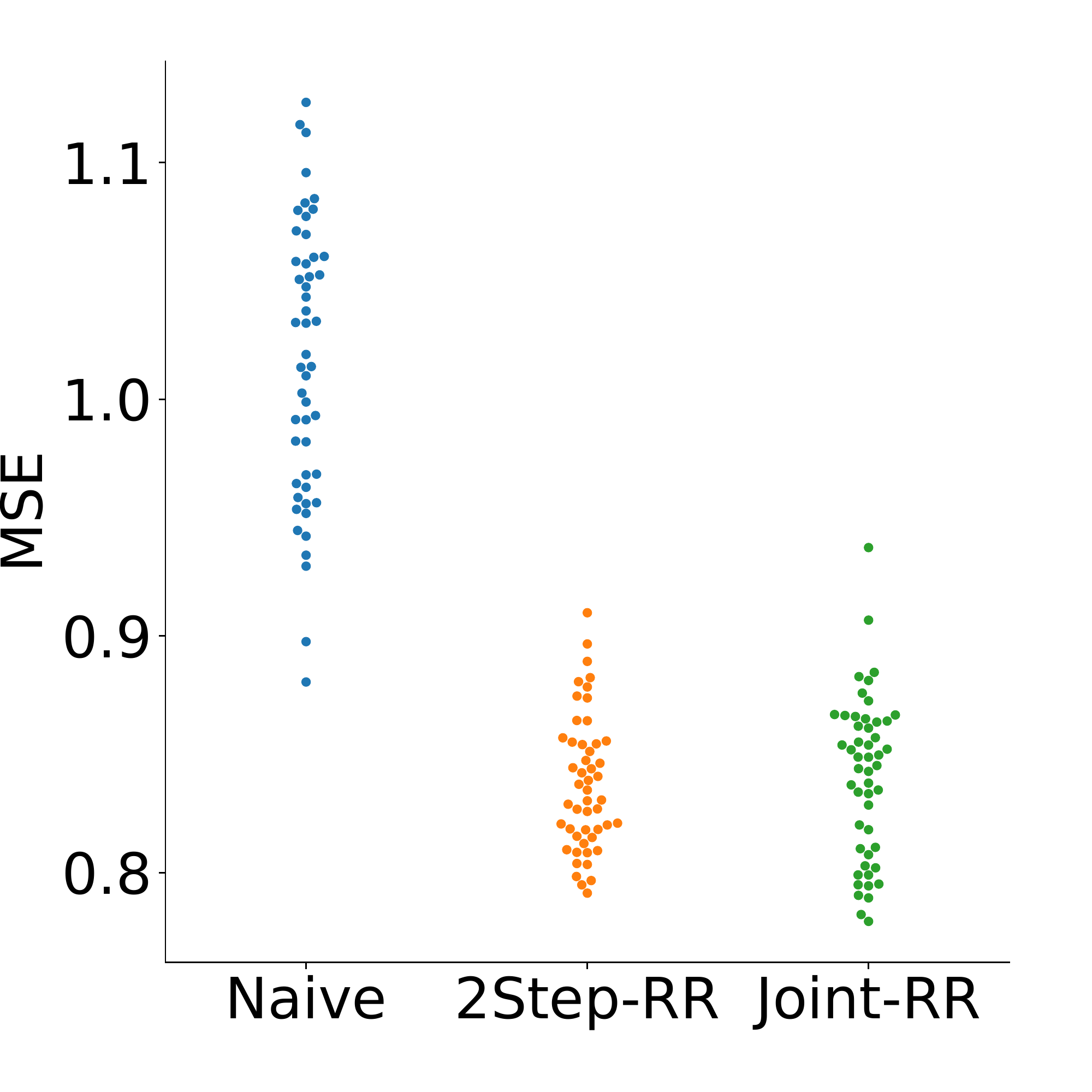}}
\subcaptionbox{Dim. 10 \label{fig:syndim2}}{\includegraphics[width=\columnwidth*23/48]{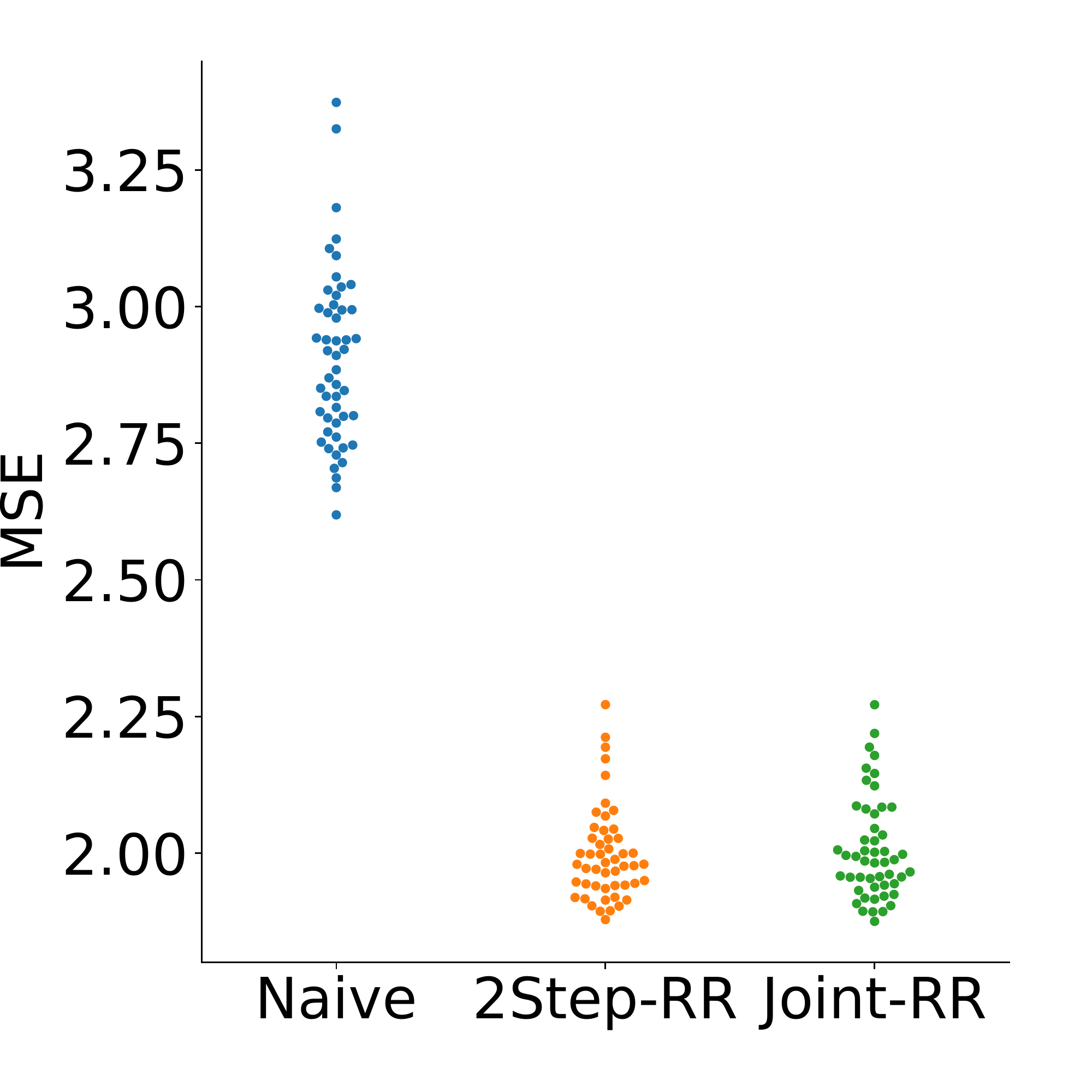}}
\subcaptionbox{Dim. 20 \label{fig:syndim2}}{\includegraphics[width=\columnwidth*23/48]{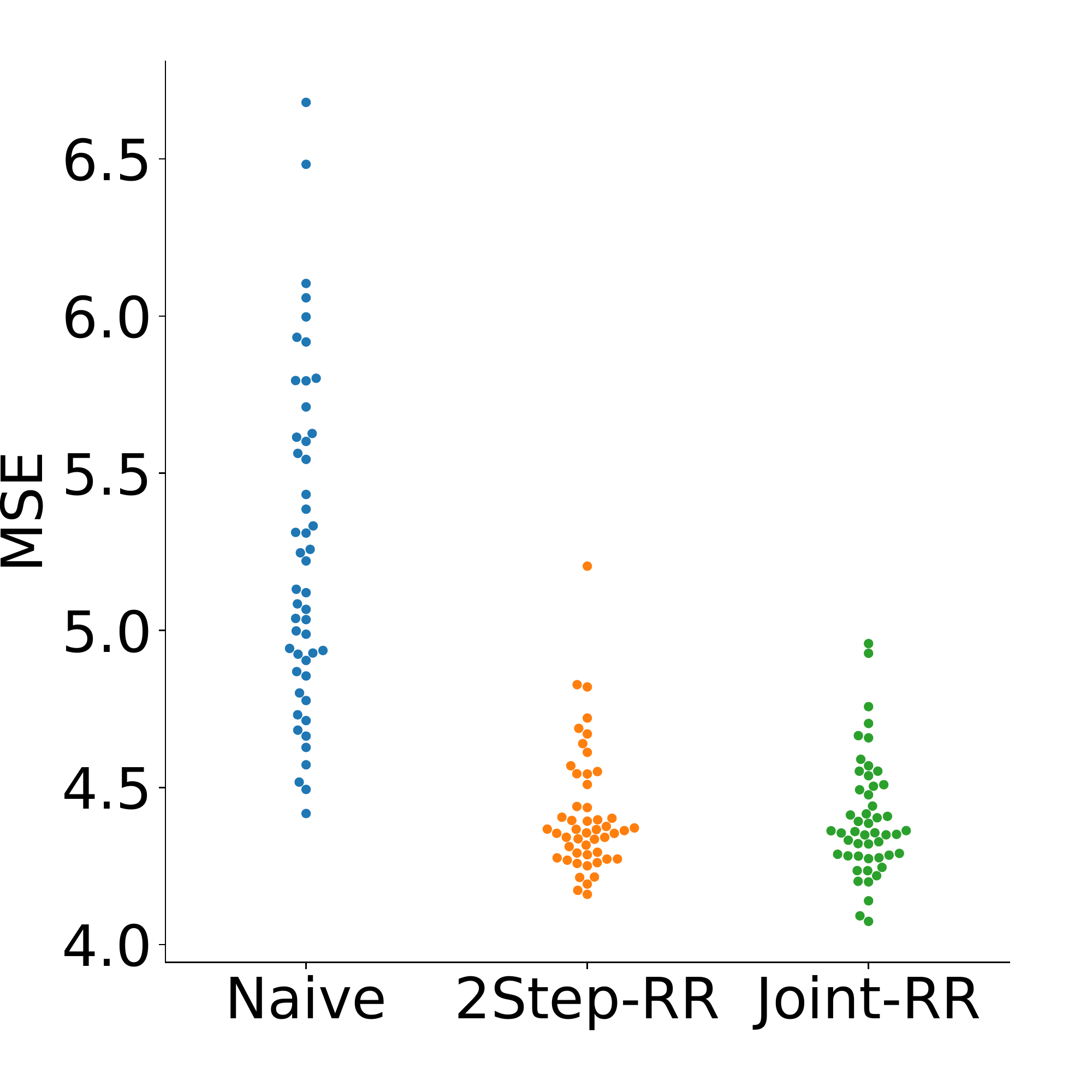}}
\caption{Experiments on synthetic data under the setting satisfying the conditional mean independence (Eq.~\eqref{eq:u_sufficient}).}
\label{fig:synthetic_satisfied}
\end{figure}

\begin{figure}[tp]
\centering
\subcaptionbox{Dim. 2 \label{fig:syndim2}}{\includegraphics[width=\columnwidth*23/48]{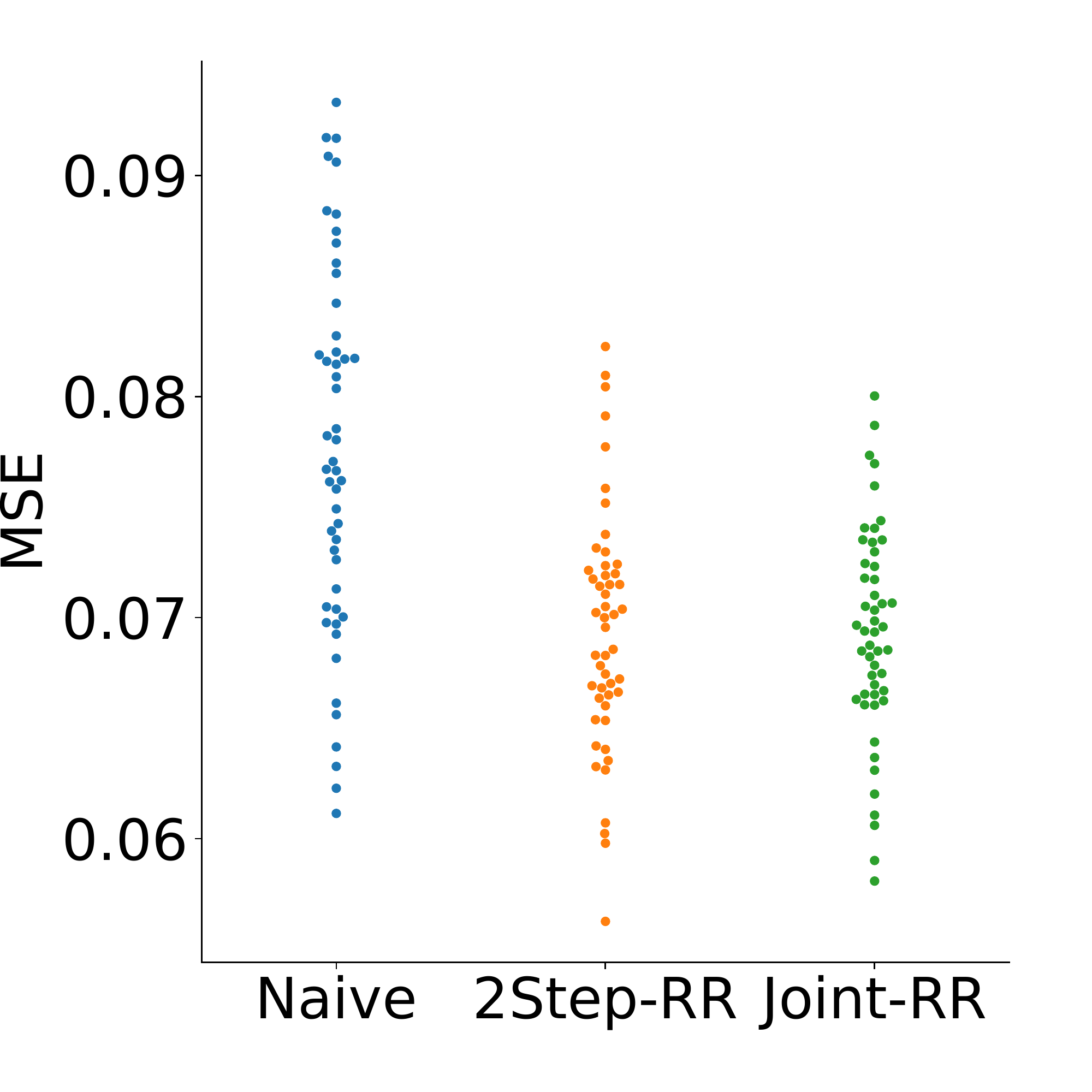}}
\subcaptionbox{Dim. 5 \label{fig:syndim2}}{\includegraphics[width=\columnwidth*23/48]{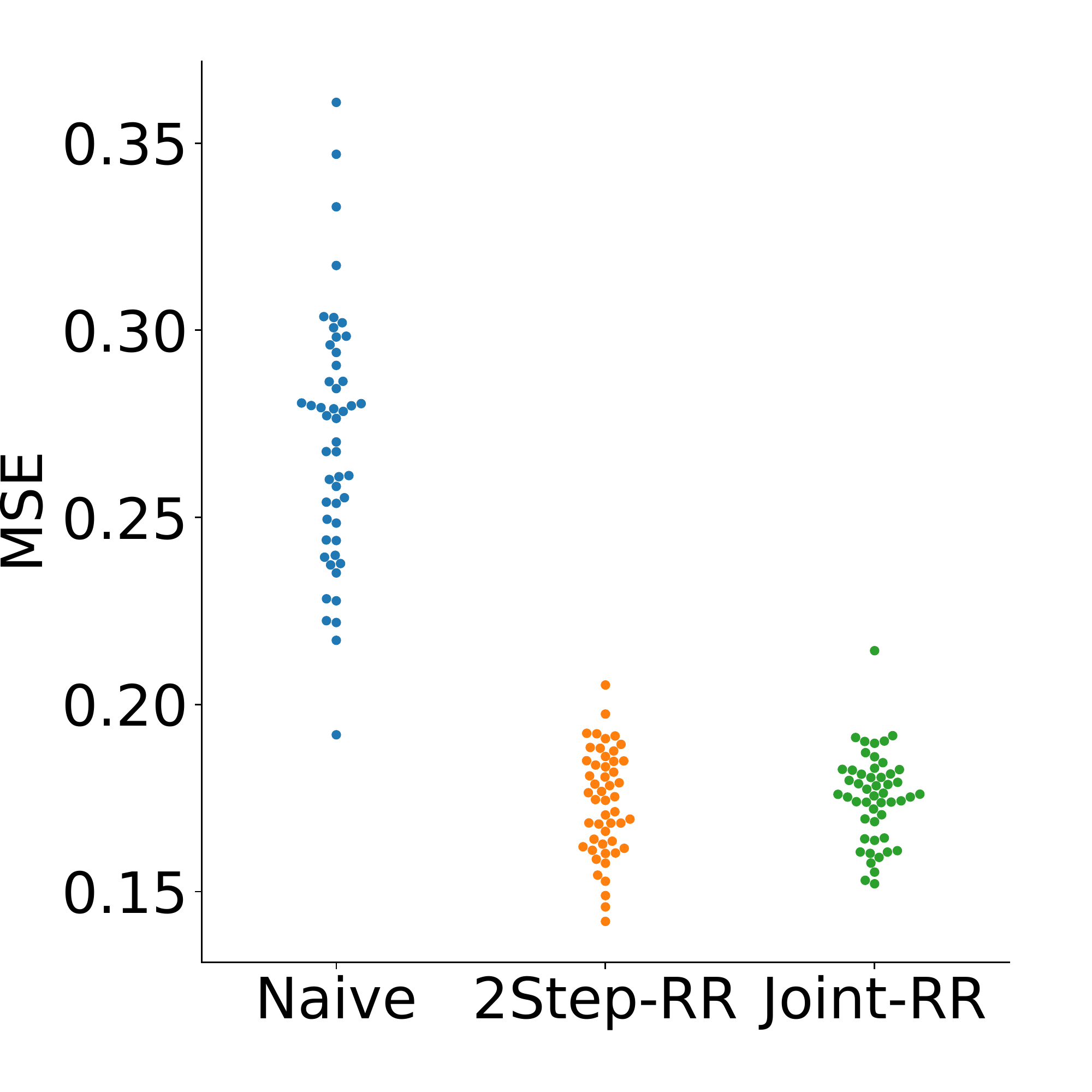}}
\subcaptionbox{Dim. 10 \label{fig:syndim2}}{\includegraphics[width=\columnwidth*23/48]{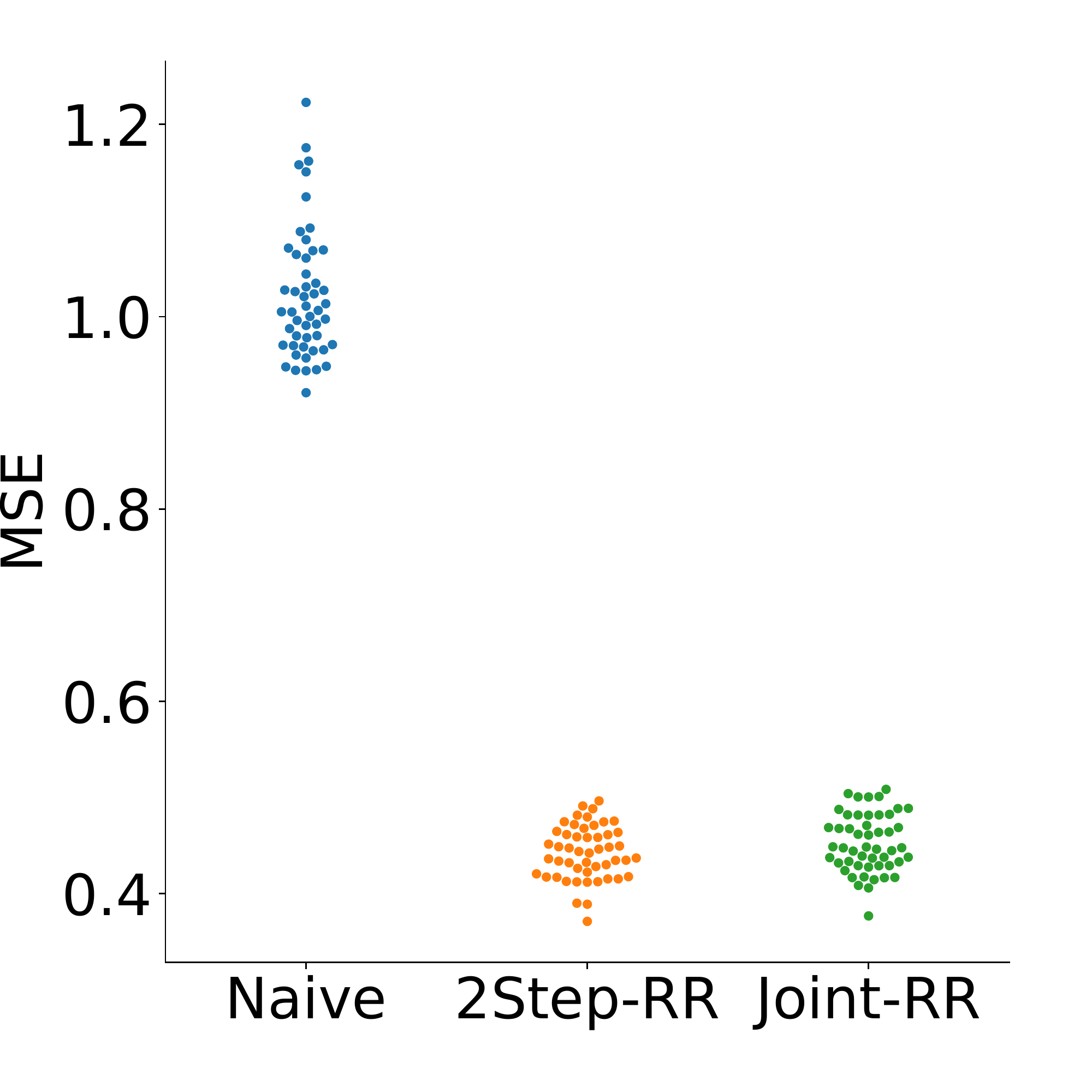}}
\subcaptionbox{Dim. 20 \label{fig:syndim2}}{\includegraphics[width=\columnwidth*23/48]{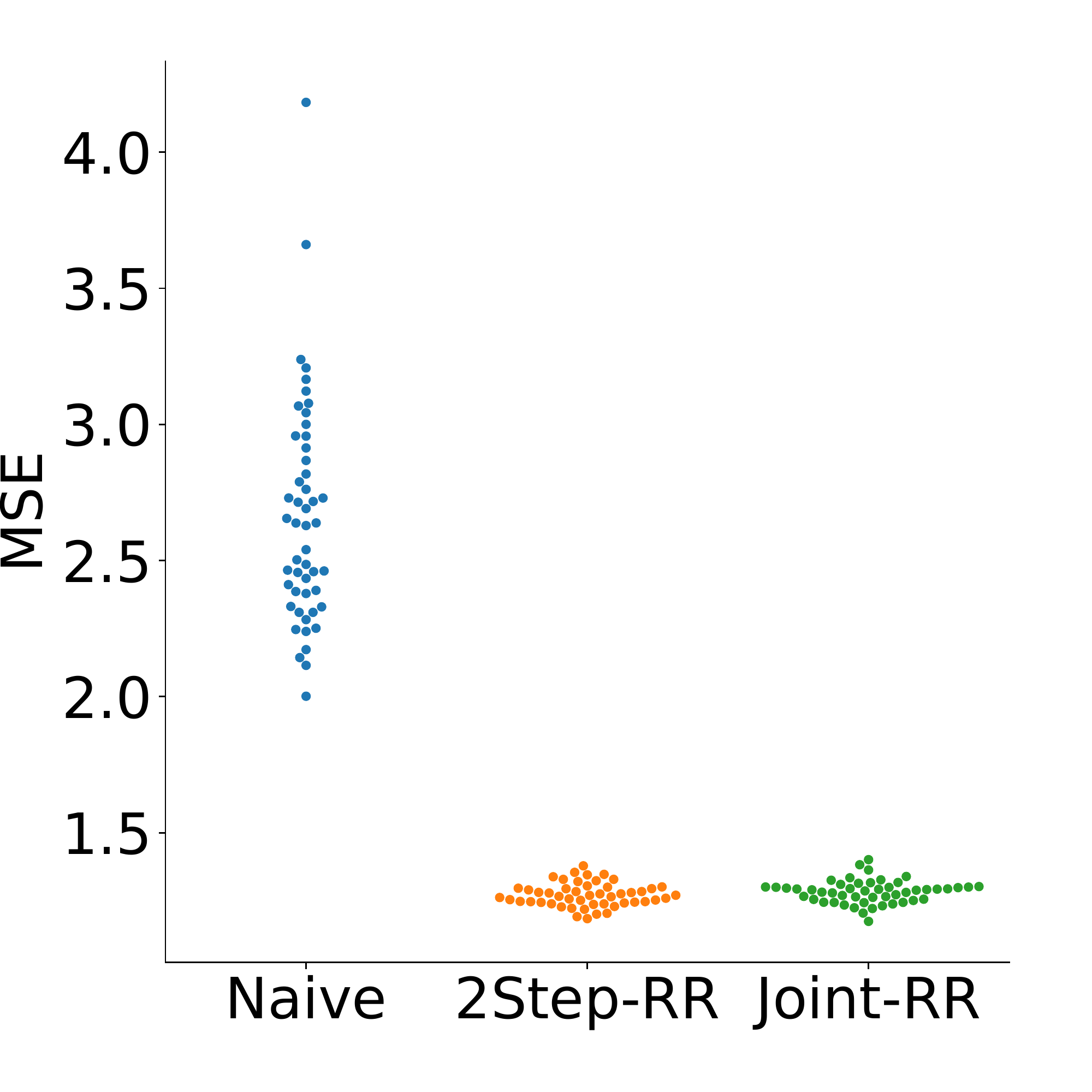}}
\caption{Experiments on synthetic data under the setting violating the conditional mean independence (Eq.~\eqref{eq:u_sufficient}).}
\label{fig:synthetic_violated}
\end{figure}

\subsection{Classification of Low-quality Images}
In this section, we test our methods in a more realistic scenario using image benchmark datasets, MNIST~\citep{mnist}, Fashion-MNIST~\citep{xiao2017fashionmnist}, CIFAR-10, and CIFAR-100~\citep{krizhevsky2009learning}.
The task here is to train a model that classifies \emph{low-quality} images using mediated uncoupled data generated from those benchmark data, where the mediating variable is high-quality images.

The motivation behind the setup is that in some applications such as on-board electronics and the Internet of Things~\citep{madakam_internet_2015},
the prediction is often done by uploading low quality images to a remote server to fit limited network bandwidths.
However, low-quality images can be difficult for human labelers to accurately label in the phase of training data collection.
Moreover, the quality of images may depend on the device, the network, and the required response time.
Instead of directly labeling low-quality images for different cases each time, we may prepare labeled high-quality images in advance only once or reuse existing data of this kind and collect pairs of high- and low-quality images in an ad hoc manner to adapt the data to each specific case.

In our experiment, we created low-quality images by down-sampling images of the benchmark datasets with
average pooling with stride \((2, 2)\) and kernel \((2, 2)\),
and we took other original images as high-quality images,
i.e.,
\(X\) is a down-sampled image, \(U\) is an image of the original resolution, and \(Y\) is a class label.
Let \(K \in \Nat\) be the number of class labels.
We use the one-hot representation for class labels, i.e., \(Y\) is \(K\)-dimensional vector with all elements being zero except for the dimension corresponding to the represented class.
In order to make models \(f\) and \(h\) output \(K\)-dimensional probability vectors,
we use the following ``square-softmax'' function for the last layers of \(f\) and \(h\):
\begin{equation*}
  \Re^{K} \in (a_1, \dots, a_K) \mapsto \frac{a_{y}^{2}}{\sum_{y'\in[K]}a_{y'}^{2}} \in \Re^{K},
\end{equation*}
where \(a_{1}, \dots, a_{K} \in \Re\) are the outputs of the second last layer.
It is similar to the standard softmax function, but it uses the square function instead of the exponential function.\footnote{In our experiments, we found that applying the square-softmax function makes training much easier than the standard softmax function.}
We use this architecture because the proposed methods are based on the squared loss whereas the softmax function typically uses cross-entropy loss.
We evaluate MSEs and the objective functions
with the squared \(\ell_{2}\)-norms \(\norm{f(x) - y}^{2}\), \(\norm{f(x) - h(u)}^{2}\), and \(\norm{h(u) - y}^{2}\)
for any \(x \in \cX\), \(f\colon\cX \to \Re^{K}\), \(h\colon\cU \to \Re^{K}\), and \(y\in\Re^{K}\),
in place of their one-dimensional versions proposed in Section~\ref{sec:problem_setting}.

We train models using mediated uncoupled data consisting of samples of \((X, U)\) and \((U, Y)\).
In the test evaluation phase, we use coupled \((X, Y)\)-data,
and we let the trained models to classify each low-quality image and compare the prediction with
the true class label with the zero-one loss.

As in the synthetic experiments, we use the three methods but with the following configurations.
\begin{itemize}
  \item For the naive method, we use a U-Net~\cite{ronneberger_u-net_2015} for predicting \(U\) from \(X\)
        and a ResNet~\citep{he_deep_2016} implemented by \citet{Idelbayev18a} for predicting \(Y\) from \(U\).
        These are considered to be state-of-the-art deep neural network architectures
        for image-to-image translation (\(X\) to \(U\)) and image classification (\(U\) to \(Y\)), respectively.
  \item For 2Step-RR, we use ResNets.
        Note that both predicting \(Y\) from \(U\) and \(Y\) from \(X\) are image classification.
  \item For Joint-RR, we again use ResNets
        since the trained models essentially are image classifiers.
        We set \(w = 0.5\).
\end{itemize}
We train all models with Adam~\citep{kingma_adam_2017} for \(200\) epochs.
We turn off the weight decay and set the other tuning parameters of Adam as in PyTorch~\citep{pytorch_NEURIPS2019}: the learning rate is \(0.001\), the \(\beta\) is \((0.9, 0.999)\).
We use randomly sampled \(10{,}000 \times 2\) mediated uncoupled data for training and \(10{,}000\) coupled \((X, Y)\)-data for test evaluation. We repeat the experiment for \(50\) times.

\begin{table*}[!h]
\centering
\caption{
  Accuracy rates and MSEs for the experiment on classification of low-quality images with the image benchmark datasets.
  The numbers outside of parentheses are means,
  and those in parentheses are standard errors
  calculated from $50$ repetitions of the experiments.
  The scores comparable to the best in terms of Wilcoxon's signed rank test are emphasized in bold fonts.
}
\label{tab:low_image}
\begin{tabular}{l|ccc|ccc}
\toprule
        & \multicolumn{3}{|c}{Accuracy} & \multicolumn{3}{|c}{MSE}\\
\midrule
Dataset & Naive & 2Step-RR & Joint-RR & Naive & 2Step-RR & Joint-RR\\
\midrule
  MNIST
        & 88.06$\%$ (1.13) & \textbf{96.19}$\%$ (0.26) & \textbf{96.34$\%$} (0.28)
        & 0.184 (0.018) & \textbf{0.059} (0.003) & \textbf{0.056} (0.003)\\
  Fashion-MNIST
        & 73.12$\%$ (0.37) & 85.53$\%$ (0.16) & \textbf{86.93$\%$} (0.14)
        & 0.417 (0.005) & 0.213 (0.002) & \textbf{0.194} (0.002)\\
  CIFAR-10
        & 48.35$\%$ (0.28) & 67.60$\%$ (0.13) & \textbf{69.10$\%$} (0.11)
        & 0.778 (0.005) & 0.444 (0.002) & \textbf{0.424} (0.001)\\
  CIFAR-100
        & 19.97$\%$ (0.12) & 27.43$\%$ (0.08) & \textbf{28.05$\%$} (0.08)
        & 0.935 (0.001) & 0.850 (0.000) & \textbf{0.846} (0.000)\\
\bottomrule
\end{tabular}
\end{table*}

\begin{figure}[tp]
\centering
\subcaptionbox{MNIST. \label{fig:MNIST_downsampled_acc}}{\includegraphics[width=\columnwidth*23/48]{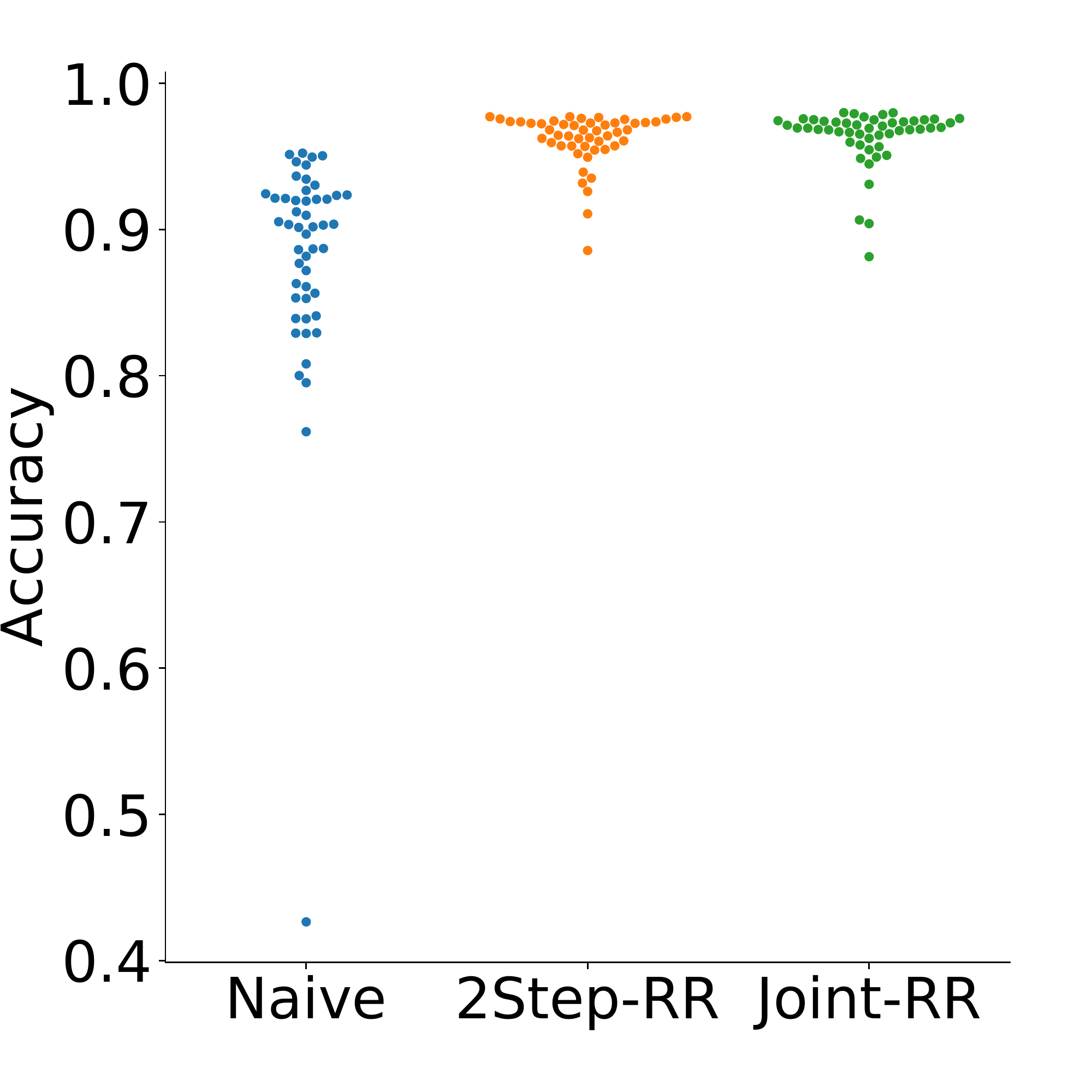}}
\subcaptionbox{Fashion-MNIST. \label{fig:FashionMNIST_downsampled_acc}}{\includegraphics[width=\columnwidth*23/48]{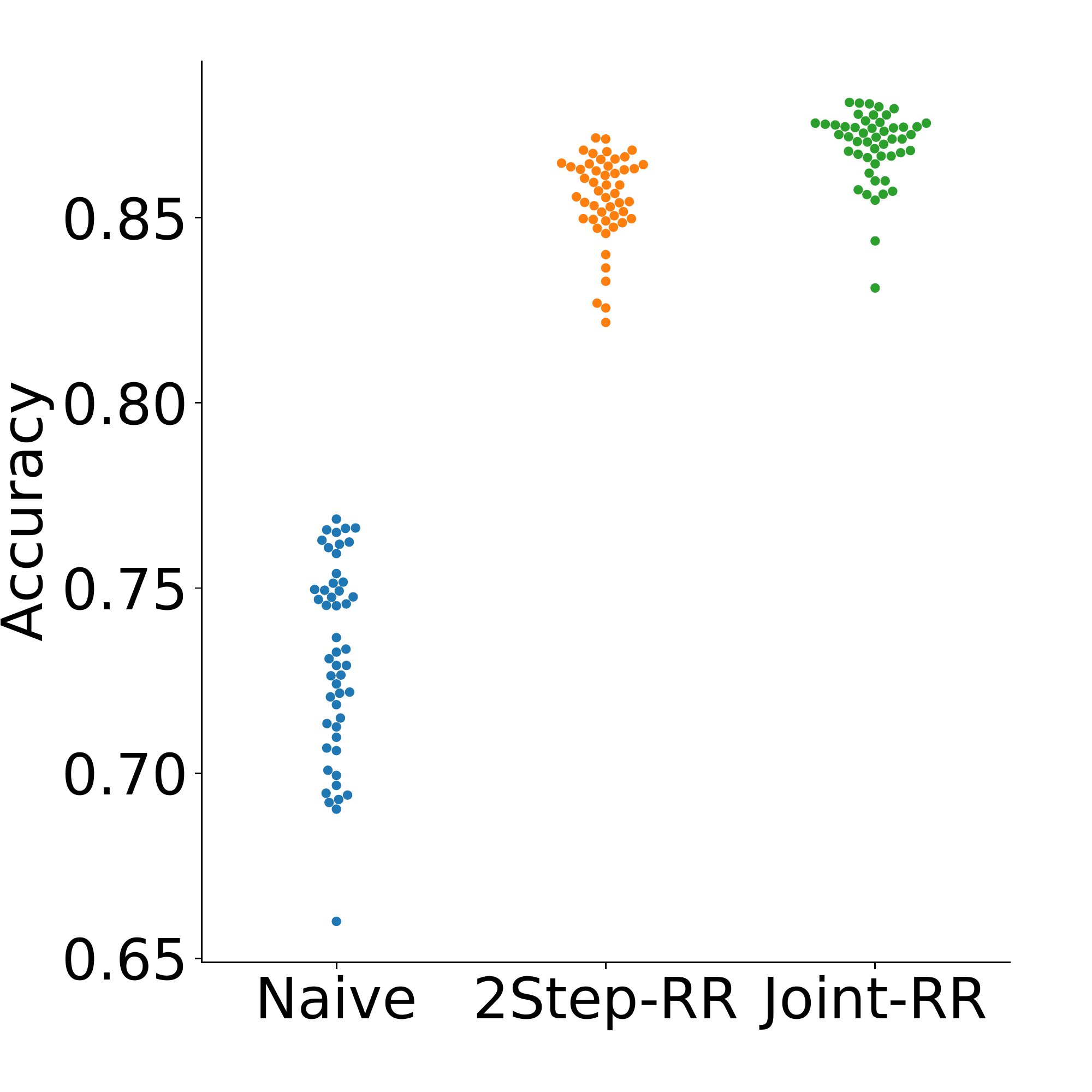}}
\subcaptionbox{CIFAR-10. \label{fig:CIFAR10_downsampled_acc}}{\includegraphics[width=\columnwidth*23/48]{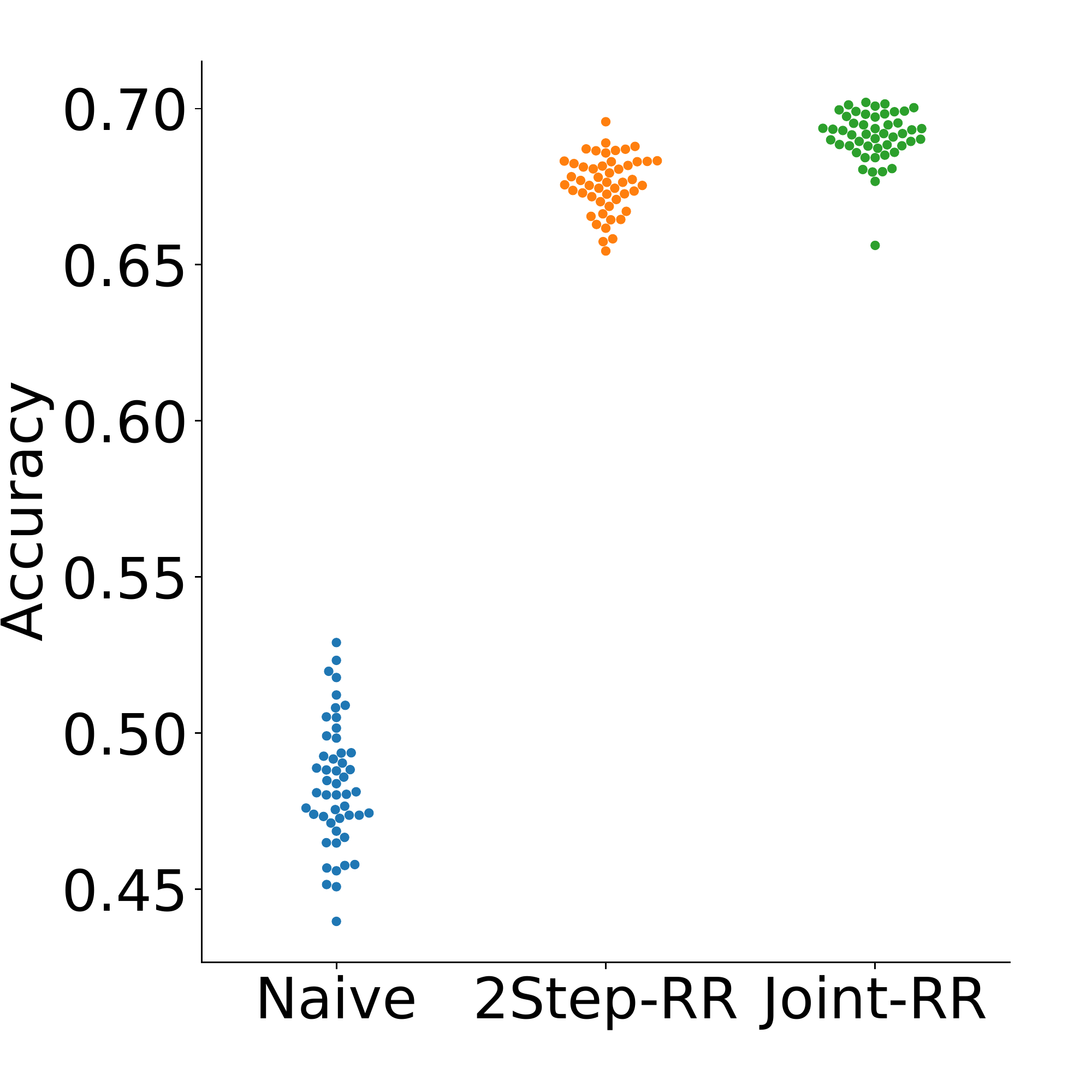}}
\subcaptionbox{CIFAR-100. \label{fig:CIFAR100_downsampled_acc}}{\includegraphics[width=\columnwidth*23/48]{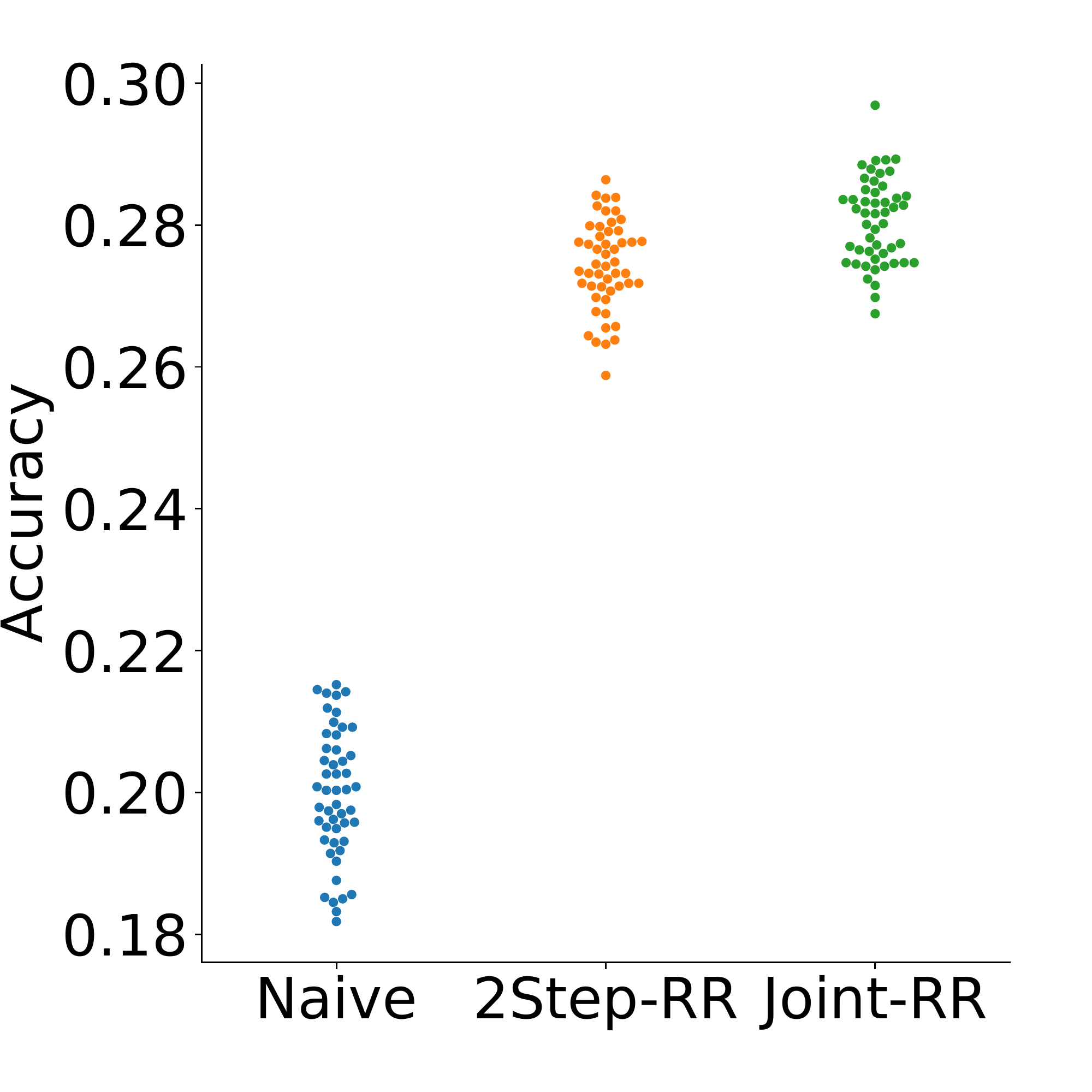}}
\caption{Accuracy rates for the experiments on low-quality image classification.}
\label{fig:ACC_downsampled}
\end{figure}

Table~\ref{tab:low_image} shows the averages and the standard errors of the accuracy rates and the MSEs obtained by each method.
In terms of accuracy, the two proposed methods outperformed the naive method.
Joint-RR improved the accuracy compared to 2Step-RR for all datasets but MNIST.
We can see the same tendency for the MSEs.
Note that these accuracy rates are far from those of state-of-the-art supervised methods (e.g., \citet{rmdl2018, foret2021sharpnessaware, fine_tuning_darts})
not only due to the uncoupled setting but also due to the down-sampling that significantly reduces the amount of information contained in images.

Figure~\ref{fig:ACC_downsampled} shows more details of the results with scatter plots of the accuracy rates.
The figure indicates that the accuracy rates for the proposed methods and the naive method are clearly isolated except for MNIST,
meaning that the proposed ones consistently performed better than the naive one for those datasets.
Similar results are observed for the MSEs (see Figure~\ref{fig:MSE_downsampled} in Appendix~\ref{sec:MSE_plot_downsampled} in the supplementary material).
We can also see that the performance of the naive method highly deviates over the trials
while those of the proposed methods tend to be more concentrated and show stable performances.
2Step-RR and Joint-RR gave comparable performance with each other,
but Joint-RR showed slightly better performance.
This performance gain may come from the regularization effect of Joint-RR (see Section~\ref{sec:UB_is_regularized}).

\section{Conclusion}
In this paper, we considered learning from mediated uncoupled data.
We proposed a method that learns a function directly predicts the target variable.
We showed an excess error bound for the proposed method
and demonstrated its practical usefulness through experiments.
This paper focused on the squared loss,
which is a standard choice for regression problems but not necessarily popular for classification.
In future work, we investigate other loss functions for classification such as the logistic loss.

\section*{Acknowledgements}
We thank Han Bao, Naoto Yokoya, Nontawat Charoenphakdee, Takashi Ishida, Yann Chevaleyre, and Yivan Zhang for the valuable discussions.
IY and MS were supported by JST CREST Grant Number JPMJCR18A2.
JH was supported by KAKENHI 21K11747.
IY and FY acknowledge the support of the ANR as part of the ``Investissements d'avenir'' program, reference ANR-19-P3IA-0001 (PRAIRIE 3IA Institute).

\bibliographystyle{icml2021}
\bibliography{main}


\clearpage
\renewcommand\thesubsection{\Alph{subsection}}
\setcounter{page}{1}
\section*{Appendix}

\subsection{Proof of Theorem~\ref{thm:ub}}
\label{apx:proof_ub_2}
We prove the following theorem.

\newtheorem*{thm:ub}{Theorem~\ref{thm:ub}}
\begin{thm:ub}
  The MSE can be bounded as
  \begin{equation}
    \E[(f(X) - Y)^2] \le J_w(f, h), \label{apxeq:ub_any_h}
  \end{equation}
  where
  \begin{align*}
    J_w(f, h)
    &\coloneqq \frac{1}{w}\E[(f(X) - h(U))^2]\\
    &\phantom{\coloneqq\ } + \frac{1}{1-w}\E[(h(U) - Y)^2]
  \end{align*}
  for any \(w \in (0, 1)\) and any \(h \in L_U^2 \coloneqq \{h\colon\mathcal{U} \to \mathcal{Y} \given \E[h(U)^2] < \infty\}\).
\end{thm:ub}

\begin{proof}
  First, from Jensen's bound, we have
  \begin{align*}
    (a + b)^{2}
    &= \lr(){w\frac{a}{w} + (1-w)\frac{a}{1-w}}^{2}\\
    &= w\lr(){\frac{a}{w}}^{2} + (1-w)\lr(){\frac{a}{1-w}}^{2}\\
    &= \frac{a^{2}}{w} + \frac{a^{2}}{1-w},
  \end{align*}
  for any \(a\in\Re\), \(b\in\Re\), and \(w \in (0, 1)\).
  Using the inequality, we obtain
  \begin{align*}
    (f(X) - Y)^{2}
    &= (f(X) - h(U) + h(U) - Y)^{2}\\
    &\le \frac{(f(X) - h(U))^{2}}{w} + \frac{(h(U) - Y)^{2}}{1-w}.
  \end{align*}
  Taking the expectations of both sides, we complete the proof.
\end{proof}

\subsection{Le Cam's Method}
We suppose all involved probability distributions have density functions.
For any density function \(p(x, u, y)\) over \(\cX \times \cU \times \cY\),
we denote its marginal distributions by \(p(x)\), \(p(u)\), \(p(y)\), \(p(x, u)\), \(p(x, y)\), and \(p(u, y)\)
and conditional distributions by \(p(u \given x)\), \(p(y \given x)\), \(p(y \given u)\), \(p(y \given u, x)\), and so on following the usual convention.

\begin{definition}
  Fix any probability density function \(p(x)\) over \(\cX\).
  For any \(c \in [0, \infty)\),
  let \(\mathbb{P}_{p, c}\) denote the set of density functions over \(\cX \times \cU \times \cY\) defined as follows.
  Any density function \(q(x, u, y)\) over \(\cX \times \cU \times \cY\) is a member of \(\mathbb{P}_{p, c}\)
  if and only if for \((X, U, Y) \sim q(x, u, y)\),
  \begin{itemize}
    \item \(X\) follows \(p(x)\), and
    \item the difference between \(\E[Y \given X]\) and \(\E[\E[Y \given U] \given X]\) is bounded by \(c\) from above in the sense of \(L^2(p)\)-distance:
          \begin{align*}
            \E[(\E[Y \given X] - \E[\E[Y \given U] \given X])^2] \le c^2.
          \end{align*}
  \end{itemize}
\end{definition}

\newcommand{\Sntil}{\Stil_{\ntil}}
\newcommand{\Sep}[1]{\operatorname{sep}[#1]}
\begin{definition}
  \label{def:sep_sample}
  For an underlying density function \(p^*(x, u, y) \in \mathbb{P}_{p, c}\),
  \emph{separate samples with proxy} is a set of random variables of the form \(S_{n,\ntil} \coloneqq ((X_1, U_1), \dots, (X_n, U_n), (\Util_1, \Ytil_1), \dots, (\Util_{\ntil}, \Ytil_{\ntil}))\),
  where \((X_1, U_1), \dots, (X_n, U_n), (\Util_1, \Ytil_1), \dots, (\Util_{\ntil}, \Ytil_{\ntil})\) are independent,
  \((X_i, U_i) \sim p^*(x, u)\), and \((\Util_i, \Ytil_i) \sim p^*(u, y)\).
  We denote the set of all possible realizations of \(S_{n, \ntil}\) by \(\mathcal{S}_{n, \ntil}\)
  and the density function of \(S_{n, \ntil}\) by \(p^*_{n, \ntil}(s)\), where
  \begin{align*}
    s
    &\equiv ((x_1, u_1), \dots, (x_n, u_n), (\xtil_1, \util_1), \dots, (\xtil_{\ntil}, \util_{\ntil}))\\
    &\in (\cX \times \cU)^{n + \ntil}.
  \end{align*}
\end{definition}

In this section, we will obtain a lower bound of the expected error that the best learner has to suffer for the worst-case instance of \(p^*(x, u, y) \in \mathbb{P}_{p, c}\) for a fixed \(c \in [0, \infty)\) and a density function \(p(x)\) over \(\cX\):
\begin{align*}
  &E_\text{minimax}\\
  &\coloneqq \inf_{\widehat{f}_{(\cdot)}\colon \mathcal{S}_{n, \ntil} \to \{f: \cX \to \cY\}}
  \sup_{p^* \in \mathbb{P}_{p, c}}
  \E[(\widehat{f}_{S_{n, \ntil}}(X) - \E[Y \given X])^2],
\end{align*}
where \((X, Y) \sim p^*(x, y)\),
and the expectation is taken over \(S_{n, \ntil}\) and \((X, Y)\).
\(\widehat{f}_{(\cdot)}\) represents a learning algorithm ranging over all mappings that input \(S_{n, \ntil}\) and output a function from \(\cX\) to \(\cY\), which include computationally intractable ones.

\begin{definition}
  Define a semi-distance metric on \(\mathbb{P}_{p, c}\), \(\rho\colon \mathbb{P}_{p, c}^2 \to [0, \infty)\), by
  \begin{align*}
    &\rho(q_1(x, u, y), q_2(x, u, y))\\
    &\coloneqq \E[(\E[Y_1 \given X] - \E[\E[Y_2 \given U_2] \given X])^2]
  \end{align*}
  for any \((q_1(x, u, y), q_2(x, u, y)) \in \mathbb{P}_{p, c}^2\),
  where \((X, U_1, Y_1) \sim q_1(x, u, y)\) and \((X, U_2, Y_2) \sim q_2(x, u, y)\).
  Note that these two tuples share the common variable \(X\) in the definition.
\end{definition}

Take any \(2\delta\)-separated density functions, \((p_1, p_2) \in \mathbb{P}_{p, c}^2\), in terms of \(\rho\): \(\rho(p_1, p_2) > 2\delta\).
Le Cam's method states that
\begin{align*}
  E_\text{minimax}
  &\ge \frac{1}{2}\delta^2 \lr(){1 - \TV(p_1, p_2})\\
  &\ge \frac{1}{2}\delta^2 \lr(){1 - \sqrt{\frac{1}{2}\KL(p_1, p_2})},
\end{align*}
where \(\TV(\cdot, \cdot)\) is the total variation distance,
and \(\KL(\cdot, \cdot)\) is the Kullback-Leibler divergence.

\subsection{Proof of Theorem~\ref{thm:proactive-two-step}}
\label{sec:proof_two_step}
The goal of this subsection is to show the following theorem.

\newtheorem*{thm:proactive-two-step}{Theorem~\ref{thm:proactive-two-step}}
\begin{thm:proactive-two-step}
  Suppose that
  \(h^*(u) \coloneqq \E[Y \given U = u] \in \cH\)
  and \(f^*(x) \coloneqq \E[h^*(U) \given X = x] \in \cF\)
  for some \(\cF \subseteq L^2_X\) and \(\cH \subseteq L^2_U\).
  Then,
  \begin{align*}
    (f^*, h^*) \in \lim_{w \toup 1} \argmin_{(f, h) \in \cF \times \cH}
    J_w(f, h) \label{eq:proposed_minus_C},
  \end{align*}
  where \(C > 0\) is a constant that does not depend on \(f\) or \(h\).
\end{thm:proactive-two-step}

Note that \((f^*, h^*)\) is the solution pair to the optimization problem of 2Step-RR with the population-level objective functionals:
\begin{align*}
  &h^* \in \argmin_{h \in \cH} \E[(h(U) - Y)^2]\\
  &\text{and}\quad
    f^* \in \argmin_{f \in \cF} \E[(f(X) - h^*(U))^2].
\end{align*}
On the other hand, the constant in Eq.~\eqref{eq:proposed_minus_C} is subtracted merely to prevent the objective value from diverging and make the solution well-defined in the limit.
Thus, the theorem states that \((f^*, h^*)\) is the limit of the solution pair to the optimization problem of the proposed method with the population-level objective functional.

\begin{definition}
  For \(w \in (0, 1)\), \(f \in L^2_X\), and \(h \in L^2_U\), define
  \begin{align*}
    Q(w, f, h)
    &\coloneqq J_w(f, h) - \frac{1}{1-w}\E[(Y - h^*(U))^2],\\
    R(w, f)
    &\coloneqq \inf_{h\in \cH} Q(w, f, h),\\
    R(1, f)
    &\coloneqq \E[(f(X) - h^*(U))^2].
  \end{align*}
\end{definition}

Recall that
\begin{align*}
  J_w(f, h)
  &\equiv \frac{1}{w} \E[(f(X) - h(U))^2]\\
  &\phantom{\equiv\ } + \frac{1}{1 - w} \E[(h(U) - Y)^2],
\end{align*}
Using the identity
\begin{align*}
  &\E[(h(U) - Y)^2]\\
  &= \E[(h(U) - \E[Y \given X])^2] + \E[(\E[Y \given X] - Y)^2]\\
  &= \E[(h(U) - h^*(U))^2] + \E[(h^*(U) - Y)^2],
\end{align*}
we obtain
\begin{align*}
  &Q(w, f, h)\\
  &= \frac{1}{w} \E[(f(X) - h(U))^2]
    + \frac{1}{1 - w} \E[(h(U) - h^*(U))^2]\\
  &\phantom{=}\ + \frac{1}{1 - w} \E[(h^*(U) - Y)^2]
    - \frac{1}{1 - w} \E[(h^*(U) - Y)^2]\\
  &= \frac{1}{w} \E[(f(X) - h(U))^2] + \frac{1}{1 - w} \E[(h(U) - h^*(U))^2].
\end{align*}
Moreover, note that
\begin{align*}
  f^* = \argmin_{f \in \cF} R(1, f).
\end{align*}

\begin{lemma}\label{lem:R_conti}
  \((w, f) \mapsto R(w, f)\) is continuous on \(\{1\} \times \cF\),
  and \(w \mapsto \inf_{f\in\cF} R(w, f)\) is continuous at \(w = 1\).
\end{lemma}

\begin{proof}
  For any \(f_0, f_1 \in \cF\) and any \(w \in (0, 1)\), we have
  \begin{align*}
    &R(w, f_1) - R(1, f_0)\\
    &= \inf_{h \in \cH} \bigg[
      \frac{1}{w}\E[(f_1(X) - h(U))^2]\\
    &\phantom{=\inf_{h \in \cH} \bigg[} + \frac{1}{1-w}\E[(h(U) - h^*(U))^2]\bigg]\\
    &\phantom{=} - \E[(f_0(X) - h^*(U))^2]\\
    &\le \frac{1}{w}\E[(f_1(X) - h^*(U))^2] - \E[(f_0(X) - h^*(U))^2].
  \end{align*}
  Since \(h = h^*\) cannot go below the infimum, we have
  \begin{align*}
    &R(w, f_1) - R(1, f_0)\\
    &\le \E[(f_1(X) - h^*(U))^2] - \E[(f_0(X) - h^*(U))^2]\\
    &\phantom{\le} + \frac{1 - w}{w}\E[(f_1(X) - h^*(U))^2]\\
    &= \E[(f_1(X) - f_0(X))(f_1(X) + f_0(X) - 2h^*(U))]\\
    &\phantom{=} + \frac{1 - w}{w}\E[(f_1(X) - h^*(U))^2]\\
    &= \E[(f_1(X) - f_0(X))(f_1(X) - f_0(X) + 2f_0(X) - 2h^*(U))]\\
    &\phantom{=} + \frac{1 - w}{w}\E[3^2(\{f_1(X) - f_0(X)\}/3 + f_0(X)/3 - h^*(U)/3)^2]\\
    &\le
      \norm{f_1 - f_0}_2(\norm{f_1 - f_0}_2 + 2\norm{f_0}_2 + 2\norm{h^*}_2)\\
    &\phantom{\le} + \frac{3(1 - w)}{w}(\norm{f_1 - f_0}_2^2 + \norm{f_0}_2^2 + \norm{h^*}_2^2)\\
    &\phantom{\le} \quad\text{(from the CauchySchwarz and Jensen's inequality)}\\
    &\to 0 \quad \text{(as \(w \toup 1\) and \(f_1 \to f_0\))}.
  \end{align*}
  On the other hand,
  \begin{align*}
    &R(w, f_1) - R(1, f_0)\\
    &= \inf_{h \in \cH} \bigg[
      \frac{w}{w^2} \E[(f_1(X) - h(U))^2]\\
    &\phantom{\inf_{h \in \cH} \bigg[\ } + \frac{1-w}{(1-w)^2}\E[(h(U) - h^*(U))^2]\bigg]\\
    &\phantom{=} - \E[(f_0(X) - h^*(U))^2]\\
    &\ge \inf_{h \in \cH} \bigg[
      \E\bigg[\bigg(w\cdot\frac{f_1(X) - h(U)}{w}\\
    &\phantom{\ge \inf_{h \in \cH} \bigg[}
      + (1-w)\cdot\frac{h(U) - h^*(U)}{1-w}\bigg)^2\bigg]\bigg]\\
    &\phantom{\ge}- \E[(f_0(X) - h^*(U))^2]
  \end{align*}
  from Jensen's inequality. Hence,
  \begin{align*}
    &R(w, f_1) - R(1, f_0)\\
    &= \E[(f_1(X) - h^*(U))^2] - \E[(f_0(X) - h^*(U))^2]\\
    &= \E[(f_1(X) - f_0(X))(f_1(X) + f_0(X) - 2h^*(U))]\\
    &\ge -\norm{f_1 - f_0}(\norm{f_1 - f_0} + 2\norm{f_0} + 2\norm{h^*})\\
    &\phantom{\le} \quad\text{(from the CauchySchwarz inequality)}\\
    &\to 0 \quad \text{(as \(f_1 \to f_0\))}.
  \end{align*}
  By the squeeze theorem, \(R(w, f_1) \to R(1, f_0)\) as \(w \toup 1\) and \(f_1 \to f_0\).

  Similarly, for any \(w \in (0, 1)\), we have
  \begin{align*}
    &\inf_{f\in\cF} R(w, f) - \inf_{f\in\cF} R(1, f)\\
    &= \inf_{f\in\cF} R(w, f) - R(1, f^*)\\
    &= \inf_{(f, h)\in\cF\times\cH} \bigg[
      \frac{1}{w}\E[(f(X) - h(U))^2]\\
    &\phantom{= \inf_{(f, h)\in\cF\times\cH} \bigg[\ }
      + \frac{1}{1-w}\E[(h(U) - h^*(U))^2]\bigg]\\
    &\phantom{=\ } - \E[(f^*(X) - h^*(U))^2]\\
    &\le \frac{1}{w}\E[(f^*(X) - h^*(U))^2]
      - \E[(f^*(X) - h^*(U))^2]\\
    &\phantom{\le} \quad\text{(since \((f, h) = (f^*, h^*)\) cannot go below the infimum)},\\
    &\le \lr(){\frac{1}{w} - 1}\E[(f^*(X) - h^*(U))^2]\\
    &\to 0 \quad \text{(as \(w \toup 1\))}.
  \end{align*}
  On the other hand,
  \begin{align*}
    &\inf_{f\in\cF} R(w, f) - \inf_{f\in\cF} R(1, f)\\
    &\ge \inf_{f\in\cF} R(w, f) - R(1, f^*)\\
    &= \inf_{(f, h)\in\cF\times\cH} \bigg[
      \frac{w}{w^2} \E[(f(X) - h(U))^2]\\
    &\phantom{= \inf_{(f, h)\in\cF\times\cH} \bigg[}
      + \frac{1-w}{(1-w)^2}\E[(h(U) - h^*(U))^2]\bigg]\\
    &\phantom{=} - \E[(f^*(X) - h^*(U))^2]\\
    &\ge \inf_{(f, h)\in\cF\times\cH}
      \E\bigg[\big(w\cdot\frac{f(X) - h(U)}{w}\\
    &\phantom{\ge \inf_{(f, h)\in\cF\times\cH}}
      + (1-w)\cdot\frac{h(U) - h^*(U)}{1-w}\big)^2\bigg]\\
    &\phantom{\ge}- \E[(f^*(X) - h^*(U))^2]
      \quad \text{(from Jensen's inequality)}\\
    &\ge \inf_{(f, h)\in\cF\times\cH}
      \E\lr[]{\lr(){f(X) - h^*(U)}^2}\\
    &\phantom{\ge \inf_{(f, h)\in\cF\times\cH}}
    - \E[(f^*(X) - h^*(U))^2]\\
    &= 0.
  \end{align*}
  By the squeeze theorem,
  \begin{equation}
    \inf_{f\in\cF} R(w, f) \to \inf_{f\in\cF} R(1, f)
  \end{equation}
  as \(w \toup 1\).
\end{proof}

\begin{lemma}
  \(f \mapsto R(1, f)\) has a \emph{well-separated} minimum \citep{vaart1998asymp}, i.e.,
  there exists a minimizer
  \begin{equation}
    f^\dagger \in \argmin_{f \in \cF} R(1, f)
  \end{equation}
  that satisfies
  \begin{equation}
    R(1, f^\dagger) < \inf_{f \in \cF, \norm{f - f^\dagger}_2 \ge \delta} R(1, f)
  \end{equation}
  for every \(\delta > 0\).
\end{lemma}
The minimizer turns out to be unique when it is well-separated, so \(f^\dagger = f^*\).

\begin{proof}
  Let
  \begin{align*}
    f^\dagger
    &\coloneqq f^*
    = \E[h^*(U) \given X = x]\\
    &\in \argmin_{f\in\cF} \E[(f(X) - h^*(U))^2].
  \end{align*}
  For any \(\delta > 0\) and any \(f \in \cF\) such that \(\norm{f - f^\dagger} \ge \delta\), we have
  \begin{align*}
    &R(1, f) - R(1, f^\dagger)\\
    &= \E[(f(X) - h^*(U))^2] - \E[(f^\dagger(X) - h^*(U))^2]\\
    &= \E[(f(X) - \E[h^*(U) \given X])^2]\\
    &\phantom{=\ } + \E[(h^*(U) - \E[h^*(U) \given X])^2]\\
    &\phantom{=\ } - \E[(\E[h^*(U) \given X] - h^*(U))^2]\\
    &= \E[(f(X) - f^\dagger(X))^2]\\
    &\ge \delta^2.
  \end{align*}
  Hence, for any \(\delta > 0\), it holds that
  \begin{equation}
    \inf_{f\in\cF,\norm{f-f^\dagger}\ge\delta} R(1, f)
    \ge R(1, f^\dagger) + \delta^2
    > R(1, f^\dagger).
  \end{equation}
\end{proof}

\begin{proof}[Proof of Theorem~\ref{thm:proactive-two-step}]
  Let
  \begin{equation}
    h_{f, w} \coloneqq \argmin_{h \in \cH} Q(w, f, h).
  \end{equation}
  First, we show that \(h_{f, w} \to h^*\) as \(w \toup 1\) for any \(f \in \cF\).
  Since \(Q(w, f, h_{w, f}) \le Q(w, f, h^*)\), we have
  \begin{align*}
    &\phantom{\le\ } \frac{1}{w}\E[(f(X) - h_{f, w}(U))^2]\\
    &\phantom{\le\ } + \frac{1}{1-w}\E[(h_{f, w}(U) - h^*(U))^2]\\
    &\le \frac{1}{w} \E[(f(X) - h^*(U))^2],
  \end{align*}
  which implies
  \begin{align*}
    &\frac{1}{1-w}\E[(h_{f, w}(U) - h^*(U))^2]\\
    &\le \frac{1}{w} \E[(f(X) - h^*(U))^2]
      - \frac{1}{w}\E[(f(X) - h_{f, w}(U))^2]\\
    &\le \frac{1}{w} \E[(f(X) - h^*(U))^2].
  \end{align*}
  Thus,
  \begin{align*}
    &\norm{h_{f, w} - h^*}_2^2
      \equiv \E[(h_{f, w}(U) - h^*(U))^2]\\
    &\le \frac{1 - w}{w} \E[(f(X) - h^*(U))^2]
      \to 0 \quad \text{(as \(w \toup 1\))}.
  \end{align*}

  Next, we show
  \begin{equation}
    f_w \coloneqq \argmin_{f \in \cF} R(w, f) \to f^*~\text{as \(w \toup 1\)}.
  \end{equation}
  From the continuity of \((w, f) \mapsto R(w, f)\) and \(w \mapsto \inf_{f \in \cF} R(w, f)\) at \(w = 1\) (Lemma~\ref{lem:R_conti}),
  for every \(\varepsilon > 0\), there exists \(\delta > 0\) such that for every \(\widetilde{w} \in (0, 1)\),
  \begin{align*}
    &\abs{\widetilde{w} - 1} < \delta\\
    &\implies \abs{R(1, f_{\widetilde{w}}) - R(\widetilde{w}, f_{\widetilde{w}})} < \varepsilon / 2\\
    &\phantom{\implies\ }~\text{and}~
      \abs{\inf_{f \in \cF} R(\widetilde{w}, f) - \inf_{f \in \cF} R(1, f)} < \varepsilon / 2\\
    &\implies \abs{R(1, f_{\widetilde{w}}) - R(1, f^*)}\\
    &\phantom{\implies\ } < \abs{R(1, f_{\widetilde{w}}) - R(\widetilde{w}, f_{\widetilde{w}})}
      + \abs{R(\widetilde{w}, f_w) - R(1, f^*)}\\ 
    &\phantom{\implies\ } < \abs{R(1, f_{\widetilde{w}}) - R(\widetilde{w}, f_{\widetilde{w}})}\\
    &\phantom{\implies\ } + \abs{\inf_{f \in \cF} R(\widetilde{w}, f) - \inf_{f \in \cF} R(1, f)}\\
    &\phantom{\implies\ } < \varepsilon\\
    &\implies
      R(1, f_{\widetilde{w}}) - R(1, f^*) < \varepsilon.
      \label{eq:eps_delta}
  \end{align*}
  On the other hand, suppose that
  \begin{align}
    &\exists \eta > 0, \forall \varepsilon > 0, \exists f \in \cF,\\
    &[\norm{f - f^*}_2 \ge \eta~\text{and}~R(1, f) - R(1, f^*) < \varepsilon]. \label{eq:suppose1}
  \end{align}
  Then,
  \begin{align*}
    &\exists \eta > 0, \forall \varepsilon > 0,
      \inf_{f \in \cF, \norm{f - f^*} \ge \eta} R(1, f) <  R(1, f^*) + \varepsilon;
  \end{align*}
  hence
  \begin{align*}
    &\exists \eta > 0,
      \inf_{f \in \cF, \norm{f - f^*}_2 \ge \eta} R(1, f) = R(1, f^*),
  \end{align*}
  which contradicts the fact that \(f^*\) is well-separated as a minimizer of \(f \mapsto R(1, f)\).
  This confirms that the negation of \eqref{eq:suppose1} holds:
  \begin{align}
    &\forall \eta > 0, \exists \varepsilon > 0,
    \forall f \in \cF, [R(1, f) - R(1, f^*) < \varepsilon\\
    &\implies \norm{f - f^*}_2 < \eta]. \label{eq:eta_eps}
  \end{align}
  Combining \eqref{eq:eps_delta} and \eqref{eq:eta_eps},
  for every \(\eta > 0\), there exist \(\varepsilon > 0\) and \(\delta > 0\) such that for every \(\widetilde{w} \in (0, 1)\),
  \begin{align*}
    \norm{\widetilde{w} - 1} < \delta
    &\implies R(1, f_{\widetilde{w}}) - R(1, f^*) < \varepsilon\\
    &\implies \norm{\widetilde{f} - f^*}_2 < \eta,
  \end{align*}
  which implies that \(w \mapsto \argmin_{f \in \cF} R(w, f)\) is continuous at \(w = 1\).

  Combining the results, we conclude that
  \begin{align*}
    (f^*, h^*) \in \lim_{w \toup 1} \argmin_{f\in\cF, h\in\cH} Q(w, f, h).
  \end{align*}
\end{proof}

\subsection{Rademacher complexity}
\label{sec:rademacher}
\begin{definition}[Rademacher Complexity]\label{def:rademacher}
  For any set of functions \(H\)
  and any probability density function $p$ over the domain of functions of \(H\),
  we define the \emph{Rademacher complexity} of \(H\) under \(p\) as
  \begin{align*}
    \mathfrak{R}_p^N(H) = \E_{v_1, \dots, v_N, \sigma_1, \dots, \sigma_N}\left[\sup_{h\in H}\frac{1}{N}\sum_{i=1}^N\sigma_ih(v_i)\right],
  \end{align*}
  where \(v_1, \dots, v_N \sim p\),
  $\sigma_1, \dots, \sigma_N$ are $\{-1, 1\}$-valued uniform random variables,
  and they are all independent.
\end{definition}

\subsection{McDiarmid's Inequality}
\label{sec:McD}
To derive a uniform deviation bound of our empirical process,
we use the following theorem called McDiarmid's inequality.
\begin{theorem}[McDiarmid's inequality]
  \label{theorem:McD}
  Let $\varphi: \mathcal{D}^N \to \Re$ be a measurable function.
  Assume that there exists a real number $B_{\varphi} > 0$ such that
  \begin{align}
    \abs{\varphi(v_1, \dots, v_N) - \varphi(v'_1,
    \dots, v'_N)} \le B_{\varphi}, \label{eq:diff_bounded}
  \end{align}
  for any $v_i, \dots, v_N, v_1, \dots, v'_N\in\mathcal{D}$ where $v_i = v'_i$
  for all but one $i\in\{1, \dots, N\}$.
  Then, for any $\mathcal{D}$-valued independent random variables $V_1,
  \dots, V_N$ and any $\delta > 0$
  the following holds with probability at least $1 - \delta$:
  \begin{align*}
    \varphi(V_1, \dots, V_N) \le \E[\varphi(V_1, \dots, V_N)] + \sqrt{\frac{B_{\varphi}^2N}{2} \log \frac{1}{\delta}}.
  \end{align*}
\end{theorem}

\subsection{Excess error bound for 2Step-RR}
\label{sec:gen_bound_pts}
\begin{proof}
  Let \(\bar{g}(x) \coloneqq \E[\widetilde{h}(U) \given X = x]\).
  Then,
  \begin{align*}
    &\E[(\widetilde{f}(X) - f^{*}(X))^{2}]\\
    &\le
      2 \E[(\widetilde{f}(X) - \bar{g}(X))^{2}]\\
    &\phantom{\le\ }
      + 2 \E[(\bar{g}(X) - f^{*}(X))^{2}].
  \end{align*}
  We are going to bound each of the terms on the right hand side.

  \paragraph{Bounding \(\E[(\widetilde{f}(X) - \bar{g}(X))^{2}]\):}
  First, we have
  \begin{align*}
    &\E[(\widetilde{f}(X) - \widetilde{h}(U))^{2}]\\
    &= \E[(\widetilde{f}(X) - \bar{g}(X))^{2}]\\
    &+ \E[(\bar{g}(X) - \widetilde{h}(U))^{2}]\\
    &+ 2\E[(\widetilde{f}(X) - \bar{g}(X))\underbrace{\E[\bar{g}(X) - \widetilde{h}(U) \given X]}_{= 0}].
  \end{align*}
  Hence,
  \begin{align*}
    &\E[(\widetilde{f}(X) - \bar{g}(X))^{2}]\\
    &= \E[(\widetilde{f}(X) - \widetilde{h}(U))^{2}]
    - \E[(\bar{g}(X) - \widetilde{h}(U))^{2}].
  \end{align*}
  Observe that
  \begin{align*}
    &\E[(\widetilde{f}(X) - \bar{g}(X))^{2}]\\
    &= \E[(\widetilde{f}(X) - \widetilde{h}(U))^{2}]
    - \E[(\bar{g}(X) - \widetilde{h}(U))^{2}]\\
    &= \E[(\widetilde{f}(X) - \widetilde{h}(U))^{2}]
    - \frac{1}{n}\sum_{{i=1}}^{n} (\bar{g}(X_{i}) - \widetilde{h}(U_{i}))^{2}\\
    &\phantom{=\ }
    + \underbrace{\frac{1}{n}\sum_{i=1}^{n}(\widetilde{f}(X_{i}) - \widetilde{h}(U_{i}))^{2}
      - \frac{1}{n}\sum_{i=1}^{n} (\bar{g}(X_{i}) - \widetilde{h}(U_{i}))^{2}}_{
        \le 0~\text{(from \(\widetilde{f}\)'s optimality)}
      }\\
    &\phantom{=\ }
    + \frac{1}{n}\sum_{i=1}^{n} (\bar{g}(X_{i}) - \widetilde{h}(U_{i}))^{2}
    - \E[(\bar{g}(X) - \widetilde{h}(U))^{2}]\\
    &\le
      2\sup_{\phi\in\cF - \cH}\left[
      \Abs{\frac{1}{n}\sum_{i=1}^{n} \phi(X_{i}, U_{i})^{2}
      - \E[\phi(X, U)^{2}]}\right],
  \end{align*}
  where \(\cF - \cH = \{(x, u) \mapsto f(x) + h(u) \given f\in\cF, h\in\cH\}\).
  Let
  \begin{align*}
    &\psi(x_{1}, u_{1}, \dots, x_{n}, u_{n})\\
    &\coloneqq
    \sup_{\phi\in\cF - \cH}\left[
      \Abs{\frac{1}{n}\sum_{i=1}^{n} \phi(x_{i}, u_{i})^{2}
      - \E[\phi(X, U)^{2}]}\right].
  \end{align*}
  Then, \(\psi\) is a function with bounded differences:
  \begin{align*}
    &\vert
      \psi(x_{1}, u_{1}, \dots, x_{j}, u_{j}, \dots, x_{n}, u_{n})\\
    &\phantom{\vert\ }
      - \psi(x_{1}, u_{1}, \dots, x'_{j}, u'_{j}, \dots, x_{n}, u_{n})
    \vert\\
    &\le
    \sup_{\phi\in\cF - \cH}\left[
      \Abs{\phi(x_{j}, u_{j})^{2} - \phi(x'_{j}, u'_{j})^{2}}
      \right]\\
    &\le
      2 (C_{\cF} + C_{\cH})^{2} / n.
  \end{align*}
  From McDiarmid's inequality for functions with bounded differences~\citep{mcdiarmid_1989},
  with probability at least \(1 - \delta\), it holds that
  \begin{align*}
    &\psi(X_{1}, U_{1}, \dots, X_{n}, U_{n})\\
    &\le \E[\psi(X_{1}, U_{1}, \dots, X_{n}, U_{n})]\\
    &\phantom{\le\ } + \sqrt{\frac{4(C_{\cF} + C_{\cH})^4}{2n}\log\frac{1}{\delta}}\\
    &= \E[\psi(X_{1}, U_{1}, \dots, X_{n}, U_{n})]\\
    &\phantom{=\ } + (C_{\cF} + C_{\cH})^2\sqrt{\frac{2}{n}\log\frac{1}{\delta}}.
  \end{align*}
  Here,
  \begin{align*}
    &\E[\psi(X_{1}, Y_{1}, \dots, X_{n}, Y_{n})\}_{{i=1}}^{n})]\\
    &\le \E\lr[]{\sup_{{\phi \in \cF - \cH}}\frac{1}{n}\Abs{\sum_{{i = 1}}^{n} \phi(X_{i}, U_{i})^{2} - \sum_{{i = 1}}^{n} \phi(X'_{i}, U'_{i})^{2}}}\\
    &\le \E\lr[]{\sup_{{h \in \mathcal{H}}}\frac{1}{n}\Abs{\sum_{{i = 1}}^{n} \sigma_{i}\phi(X_{i}, U_{i}))^{2}}}\\
    &\le \mathfrak{R}_{n}(\{\phi^{2} \given \phi \in \cF - \cH\})\\
    &\le 2 (C_{\cF} + C_{\cH}) (\mathfrak{R}_{n}(\cF) + \mathfrak{R}_{n}(\cH)),
  \end{align*}
  where $\sigma_1, \dots, \sigma_N$ are independent $\{-1, 1\}$-valued uniform random variables.
  The last inequality follows from Talagrand's contraction lemma~\citep{ledoux_probability_2011}
  Combining what we have obtained, we confirm that
  \(\E[(\widetilde{f}(X) - \bar{g}(X))^{2}]\)
  can be controlled by the Rademacher complexities of \(\cF\) and \(\cH\):
  \begin{align*}
    &\E[(\widetilde{f}(X) - \bar{g}(X))^{2}]\\
    &\le 2 \psi(X_{1}, U_{1}, \dots, X_{n}, U_{n})\\
    &\le 2 \E[\psi(X_{1}, U_{1}, \dots, X_{n}, U_{n})]\\
    &\phantom{\le\ } + 2 (C_{\cF} + C_{\cH})^2\sqrt{\frac{2}{n} \log\frac{1}{\delta}}\\
    &\le 4 (C_{\cF} + C_{\cH}) (\mathfrak{R}_{n}(\cF) + \mathfrak{R}_{n}(\cH))\\
    &\phantom{\le\ } + 2 (C_{\cF} + C_{\cH})^2\sqrt{\frac{2}{n} \log\frac{1}{\delta}}.
  \end{align*}

  \paragraph{Bounding \(\E[(\bar{g}(X) - f^{*}(X))^{2}]\):}
  First, we are going to show that \(\E[(\bar{g}(X) - f^{*}(X))^{2}]\)
  can be bounded in terms of \(\E[(\widetilde{h}(U) - h^{*}(U))^{2}]\).
  From the optimality of \(\bar{g}\), we have
  \begin{equation}
    \E[(\bar{g}(X) - \widetilde{h}(U))^{2}]
    \le \E[(f^{*}(X) - \widetilde{h}(U))^{2}].
  \end{equation}
  By re-arranging equations, we get
  \begin{align*}
    &\E[(\bar{g}(X) - f^{*}(X))^{2}]\\
    &\le 2 \E[(\widetilde{h}(U) - h^{*}(U))(\bar{g}(X) - f^{*}(X))]\\
    &\le 2 \sqrt{\E[(\widetilde{h}(U) - h^{*}(U))^{2}]\E[(\bar{g}(X) - f^{*}(X))^{2}]},
  \end{align*}
  where the last inequality follows from the Cauchy-Schwarz inequality.
  This implies
  \begin{equation}
    \E[(\bar{g}(X) - f^{*}(X))^{2}]
    \le 4 \E[(\widetilde{h}(U) - h^{*}(U))^{2}].
  \end{equation}
  Next, we bound \(\E[(\widetilde{h}(U) - h^{*}(U))^{2}]\) using a standard generalization error bound
  using a uniform deviation bound and the Rademacher complexity.
  Let
  \begin{align*}
    &\varphi(\{(u_{i}, y_{i})\}_{{i=1}}^{n'}; \mathcal{H})\\
    &\coloneqq \sup_{{h \in \mathcal{H}}}\Abs{\frac{1}{n'}\sum_{{i = 1}}^{n'} (h(u_{i}) - y_{i})^{2} - \E[(h(U) - Y)^{2}]}.
  \end{align*}
  Let \(\{(u_{i}, y_{i})\}_{i=1}^{n'} \subseteq (\mathcal{U}\times \mathcal{Y})^{n'}\) and \(\{(u'_{i}, y'_{i})\}_{i=1}^{n'} \subseteq (\mathcal{U}\times \mathcal{Y})^{n'}\)
  be any two sets of size \(n'\) that differ from each their only by one pair of elements,
  \(((u_{\iota}, y_{\iota}), (u'_{\iota}, y'_{\iota}))\).
  One can show that
  \begin{align*}
    &\Abs{\varphi(\{(u_{i}, y_{i})\}_{{i=1}}^{n'}; \mathcal{H})
      - \varphi(\{(u'_{i}, y'_{i})\}_{{i=1}}^{n'}; \mathcal{H})}\\
    &\le \sup_{{h\in\mathcal{H}}}\abs{(h(u_{\iota}) - y_{\iota})^2 - (h(u'_{\iota}) - y'_{\iota})^2}\\
    &\le 2(C_{\cH} + C_{\cY})^2 / n'.
  \end{align*}
  From McDiarmid's inequality for functions with bounded differences~\citep{mcdiarmid_1989} and the union bound,
  with probability at least \(1 - \delta\), it holds that
  \begin{align*}
    &\varphi(\{(U'_{i}, Y'_{i})\}_{{i=1}}^{n'}; \mathcal{H})\\
    &\le \E[\varphi(\{(U'_{i}, U'_{i})\}_{{i=1}}^{n'}; \mathcal{H})]\\
    &\phantom{\le\ } + \sqrt{\frac{4(C_{\cH} + C_{\cY})^4}{2n'}\log\frac{1}{\delta}}\\
    &= \E[\varphi(\{(U'_{i}, U'_{i})\}_{{i=1}}^{n'}; \mathcal{H})]\\
    &\phantom{\le\ } + (C_{\cH} + C_{\cY})^2\sqrt{\frac{2}{n'}\log\frac{1}{\delta}}.
  \end{align*}
  Here,
  \begin{align*}
    &\E[\varphi(\{(U'_{i}, U'_{i})\}_{{i=1}}^{n'}; \mathcal{H})]\\
    &\le \E\lr[]{\sup_{{h \in \mathcal{H}}}\frac{1}{n'}\Abs{\sum_{{i = 1}}^{n'} (h(U'_{i}) - Y'_{i})^{2} - \sum_{{i = 1}}^{n'}  (h(U'_{i}) - Y'_{i})^{2}}}\\
    &\le \E\lr[]{\sup_{{h \in \mathcal{H}}}\frac{1}{n'}\Abs{\sum_{{i = 1}}^{n'} \sigma_{i}(h(U'_{i}) - Y'_{i})^{2}}}\\
    &\le \E\lr[]{\sup_{{h \in \mathcal{H}}}\frac{1}{n'}\Abs{\sum_{{i = 1}}^{n'} \sigma_{i}(h(U'_{i}) - Y'_{i})^{2}}}\\
    &\le \mathfrak{R}_{n'}(\{(u, y) \mapsto (h(u) - y)^{2} \given h \in \mathcal{H}\})\\
    &\le 2 (C_{\mathcal{H}} + C_{\mathcal{Y}}) \mathfrak{R}_{n'}(\mathcal{H}),
  \end{align*}
  where $\sigma_1, \dots, \sigma_N$ are independent $\{-1, 1\}$-valued uniform random variables.
  The last inequality follows from Talagrand's contraction lemma~\citep{ledoux_probability_2011}.
  Thus, we have
  \begin{align*}
    &\varphi(\{(U'_{i}, Y'_{i})\}_{{i=1}}^{n'}; \mathcal{H})\\
    &\equiv \sup_{{h \in \mathcal{H}}}\Abs{\sum_{{i = 1}}^{n'} (h(u_{i}) - y_{i})^{2} - \E[(h(U) - Y)^{2}]}\\
    &\le 2 (C_{\mathcal{H}} + C_{\mathcal{Y}}) \mathfrak{R}_{n'}(\mathcal{H})
    + (C_{\cH} + C_{\cY})^2\sqrt{\frac{2}{n'}\log\frac{1}{\delta}}.
  \end{align*}
  Also, because
  \begin{align*}
    &\E[(\widetilde{h}(U) - Y)^{2}]\\
    &= \E[(\widetilde{h}(U) - h^{*}(U))^{2}] + \E[(h^{*}(U) - Y)^{2}],
  \end{align*}
  we get
  \begin{align*}
    &\E[(\widetilde{h}(U) - h^{*}(U))^{2}]\\
    &= \E[(\widetilde{h}(U) - Y)^{2}] - \E[(h^{*}(U) - Y)^{2}].
  \end{align*}
  Using the results above, we have
  \begin{align*}
    &\E[(\widetilde{h}(U) - h^{*}(U))^{2}]\\
    &= \E[(\widetilde{h}(U) - Y)^{2}] - \E[(h^{*}(U) - Y)^{2}]\\
    &= \E[(\widetilde{h}(U) - Y)^{2}]
      - \frac{1}{n'}\sum_{{i=1}}^{n'}[(h(U'_{i}) - Y'_{i})^{2}]\\
    &+ \frac{1}{n'}\sum_{{i=1}}^{n'}[(\widetilde{h}(U'_{i}) - Y'_{i})^{2}]
      - \frac{1}{n'}\sum_{{i=1}}^{n'}[(h^{*}(U'_{i}) - Y'_{i})^{2}]\\
    &+ \frac{1}{n'}\sum_{{i=1}}^{n'}[(h^{*}(U'_{i}) - Y'_{i})^{2}]
      - \E[(h^{*}(U) - Y)^{2}]\\
    &\le 2 (C_{\mathcal{H}} + C_{\mathcal{Y}}) \mathfrak{R}_{n'}(\mathcal{H})
    + (C_{\cH} + C_{\cY})^2\sqrt{\frac{2}{n'}\log\frac{1}{\delta}}\\
    &+ 0\\
    &+ 2 (C_{\mathcal{H}} + C_{\mathcal{Y}}) \mathfrak{R}_{n'}(\mathcal{H})
    + (C_{\cH} + C_{\cY})^2\sqrt{\frac{2}{n'}\log\frac{1}{\delta}}\\
    &\le 4 (C_{\mathcal{H}} + C_{\mathcal{Y}}) \mathfrak{R}_{n'}(\mathcal{H})
    + (C_{\cH} + C_{\cY})^2\sqrt{\frac{2}{n'}\log\frac{1}{\delta}}.
  \end{align*}

  \paragraph{Bounding \(\E[(\widetilde{f}(X) - f^*(X))^2]\):}
  Finally, we summarizing the results above to obtain
  \begin{align*}
    &\E[(\widetilde{f}(X) - f^{*}(X))^{2}]\\
    &\le
      2 \E[(\widetilde{f}(X) - \bar{g}(X))^{2}]\\
    &\phantom{\le\ }
      + 2 \E[(\bar{g}(X) - f^{*}(X))^{2}]\\
    &\le 8 (C_{\cF} + C_{\cH}) (\mathfrak{R}_{n}(\cF) + \mathfrak{R}_{n}(\cH))\\
    &\phantom{\le\ } + 4 (C_{\cF} + C_{\cH})^2\sqrt{\frac{2}{n} \log\frac{1}{\delta}}\\
    &\phantom{\le\ } + 8 (C_{\mathcal{H}} + C_{\mathcal{Y}}) \mathfrak{R}_{n'}(\mathcal{H})\\
    &\phantom{\le\ } + 2 (C_{\cH} + C_{\cY})^2\sqrt{\frac{2}{n'}\log\frac{1}{\delta}}\\
    &\le \mathcal{O}_{p}\bigg(
      \mathfrak{R}_{n}(\cF)
      + \mathfrak{R}_{n}(\cH)
      + \mathfrak{R}_{n'}(\cH)\\
    &\phantom{\le\mathcal{O}_{p}\bigg(\ } + \sqrt{\lr(){\frac{1}{n} + \frac{1}{n'}}\log\frac{1}{\delta}}
      \bigg).
  \end{align*}
\end{proof}

\subsubsection{Excess risk bound for Joint-RR}
\label{sec:gen_bound_joint}
\begin{theorem}\label{thm:gen_bound_joint}
  Let \(C_{\cF} \coloneqq \sup_{{f\in\cF, x\in\cX}}f(x) < \infty\),
  \(C_{\cH} \coloneqq \sup_{{h\in\cH, u\in\cU}}f(u) < \infty\),
  and \(C_{\cY} \coloneqq \sup \cY < \infty\).
  Let \(\mathfrak{R}_{n}(\cF)\) denote the Rademacher complexity of \(\cF\)
  over \(\{(X_{i}, U_{i})\}_{i=1}^{n}\)
  and \(\mathfrak{R}_{n'}(\cH)\) denote that of \(\cH\)
  over \(\{(U'_{i}, Y'_{i})\}_{i=1}^{n'}\) (see the exact definitions in Appendix~\ref{sec:rademacher} in the supplementary material).
  Let $(f_{w}^{*}, h_{w}^{*}) \in \argmin_{(f, h) \in \cF \times \cH} J_{w}(f, h)$.
  Suppose that \(h_{w}^{*} \in \cH\) and \(f_{w}^{*} \in \cF\).
  Then, the excess risk can be bounded with probability at least \(1 - \delta\) as
  \begin{align*}
    &J_{w}(\widehat{f}_{w}, \widehat{h}_{w}) - J_{w}(f_{w}^{*}, h_{w}^{*})\\
    &\le \frac{8 (C_{\cF} + C_{\cH})}{w} (\mathfrak{R}_{n}(\cF) + \mathfrak{R}_{n}(\cH))\\
    &\phantom{\le {}} + \frac{4 (C_{\mathcal{H}} + C_{\mathcal{Y}})}{1-w} \mathfrak{R}_{n'}(\mathcal{H})\\
    &\phantom{\le {}} + \frac{4 (C_{\cF} + C_{\cH})^2}{w}\sqrt{\frac{2}{n} \log\frac{1}{\delta}}\\
    &\phantom{\le {}} + \frac{2(C_{\cH} + C_{\cY})^2}{1-w}\sqrt{\frac{2}{n'}\log\frac{1}{\delta}}\\
    &\le \mathcal{O}_{p}\bigg(
      \mathfrak{R}_{n}(\cF)
      + \mathfrak{R}_{n}(\cH)
      + \mathfrak{R}_{n'}(\cH)
      + \frac{1}{\sqrt{n}} + \frac{1}{\sqrt{n'}}\bigg).
  \end{align*}
\end{theorem}

\begin{proof}
  \begin{align}
    &J_{w}(\widehat{f}_{w}, \widehat{h}_{w}) - J_{w}(f_{w}^{*}, h_{w}^{*})\\
    &= J_{w}(\widehat{f}_{w}, \widehat{h}_{w}) - \widehat{J}_{w}(\widehat{f}_{w}, \widehat{h}_{w})\\
    &\phantom{={}} + \underbrace{\widehat{J}_{w}(\widehat{f}_{w}, \widehat{h}_{w}) - \widehat{J}_{w}(f_{w}^{*}, h_{w}^{*})}_{\le 0}\\
    &\phantom{={}} + \widehat{J}_{w}(f_{w}^{*}, h_{w}^{*}) - J_{w}(f_{w}^{*}, h_{w}^{*})\\
    &\le 2\sup_{(f, h) \in \mathcal{F}\times\mathcal{H}} \Abs{J_{w}(f, h) - \widehat{J}_{w}(f, h)}. \label{eq:diffJ_le_sup}
  \end{align}
  Here,
  \begin{align*}
    &\sup_{(f, h) \in \mathcal{F}\times\mathcal{H}} \Abs{J_{w}(f, h) - \widehat{J}_{w}(f, h)}\\
    &\le \frac{1}{w} \sup_{(f, h) \in \mathcal{F}\times\mathcal{H}} \Abs{\E[(f(X) - h(U))^{2}] - \frac{1}{n} \sum_{i=1}^{n} (f(X_{i}) - h(U_{i}))^{2}}\\
    &\phantom{\le {}} + \frac{1}{1-w} \sup_{(f, h) \in \mathcal{F}\times\mathcal{H}} \Abs{\E[(h(U) - Y)^{2}] - \frac{1}{n'} \sum_{i=1}^{n'}(h(U_{i}') - Y_{i}')^{2}}.
  \end{align*}
  Let
  \begin{align*}
    &\psi(x_{1}, u_{1}, \dots, x_{n}, u_{n})\\
    &\coloneqq \sup_{(f, h) \in \mathcal{F}\times\mathcal{H}} \Abs{\E[(f(X) - h(U))^{2}] - \frac{1}{n} \sum_{i=1}^{n} (f(x_{i}) - h(u_{i}))^{2}}.
  \end{align*}
  Then, \(\psi\) is a function with bounded differences:
  \begin{align*}
    &\vert
      \psi(x_{1}, u_{1}, \dots, x_{j}, u_{j}, \dots, x_{n}, u_{n})\\
    &\phantom{\vert\ }
      - \psi(x_{1}, u_{1}, \dots, x'_{j}, u'_{j}, \dots, x_{n}, u_{n})
    \vert\\
    &\le
    \sup_{(f, h)\in\cF\times\cH}\left[
      \Abs{(f(x_{j}) - h(u_{j}))^{2} - (f(x'_{j}) - h(u'_{j}))^{2}}
      \right]\\
    &\le
      2 (C_{\cF} + C_{\cH})^{2} / n.
  \end{align*}
  From McDiarmid's inequality for functions with bounded differences~\citep{mcdiarmid_1989},
  with probability at least \(1 - \delta\), it holds that
  \begin{align*}
    &\psi(X_{1}, U_{1}, \dots, X_{n}, U_{n})\\
    &\le \E[\psi(X_{1}, U_{1}, \dots, X_{n}, U_{n})]\\
    &\phantom{\le\ } + \sqrt{\frac{4(C_{\cF} + C_{\cH})^4}{2n}\log\frac{1}{\delta}}\\
    &= \E[\psi(X_{1}, U_{1}, \dots, X_{n}, U_{n})]\\
    &\phantom{=\ } + (C_{\cF} + C_{\cH})^2\sqrt{\frac{2}{n}\log\frac{1}{\delta}}.
  \end{align*}
  Here,
  \begin{align*}
    &\E[\psi(X_{1}, Y_{1}, \dots, X_{n}, Y_{n})\}_{{i=1}}^{n})]\\
    &\le \E\Bigg[\sup_{{(f, h) \in \cF\times\cH}}\frac{1}{n}\Bigg\vert \sum_{{i = 1}}^{n} (f(X_{i}) - h(U_{i}))^{2}\\
    &\phantom{\le \E\Bigg[\sup_{{(f, h) \in \cF\times\cH}}\frac{1}{n}\Bigg\vert}
      - \sum_{{i = 1}}^{n} (f(X'_{i}) - h(U'_{i}))^{2}\Bigg\vert\Bigg]\\
    &\le \E\lr[]{\sup_{{h \in \mathcal{H}}}\frac{1}{n}\Abs{\sum_{{i = 1}}^{n} \sigma_{i}(f(X_{i}) - h(U_{i}))^{2}}}\\
    &\le \mathfrak{R}_{n}(\{(x, u) \mapsto (f(x) - h(u))^{2} \given (f, h) \in \cF \times \cH\})\\
    &\le 2 (C_{\cF} + C_{\cH}) (\mathfrak{R}_{n}(\cF) + \mathfrak{R}_{n}(\cH)),
  \end{align*}
  where $\sigma_1, \dots, \sigma_N$ are independent $\{-1, 1\}$-valued uniform random variables.
  The last inequality follows from the Ledoux-Talagrand contraction lemma~\citep{ledoux_probability_2011}.
  Thus,
  \begin{align*}
    &\sup_{(f, h) \in \mathcal{F}\times\mathcal{H}} \bigg\vert \E[(f(X) - h(U))^{2}]\\
    &\phantom{\sup_{(f, h) \in \mathcal{F}\times\mathcal{H}} \vert} - \frac{1}{n} \sum_{i=1}^{n} (f(X_{i}) - h(U_{i}))^{2} \bigg\vert\\
    &\le 4 (C_{\cF} + C_{\cH}) (\mathfrak{R}_{n}(\cF) + \mathfrak{R}_{n}(\cH))\\
    &\phantom{\le\ } + 2 (C_{\cF} + C_{\cH})^2\sqrt{\frac{2}{n} \log\frac{1}{\delta}}.
  \end{align*}
  Similarly, we have
  \begin{align*}
    &\sup_{h \in \cH}\Abs{\sum_{{i = 1}}^{n'} (h(u_{i}) - y_{i})^{2} - \E[(h(U) - Y)^{2}]}\\
    &\le 2 (C_{\mathcal{H}} + C_{\mathcal{Y}}) \mathfrak{R}_{n'}(\mathcal{H})\\
    &\phantom{\le {}} + (C_{\cH} + C_{\cY})^2\sqrt{\frac{2}{n'}\log\frac{1}{\delta}}.
  \end{align*}
  From Eq.~\eqref{eq:diffJ_le_sup}, we get
  \begin{align*}
    &J_{w}(\widehat{f}_{w}, \widehat{h}_{w}) - J_{w}(f_{w}^{*}, h_{w}^{*})\\
    &\le \frac{8 (C_{\cF} + C_{\cH})}{w} (\mathfrak{R}_{n}(\cF) + \mathfrak{R}_{n}(\cH))\\
    &\phantom{\le {}} + \frac{4 (C_{\mathcal{H}} + C_{\mathcal{Y}})}{1-w} \mathfrak{R}_{n'}(\mathcal{H})\\
    &\phantom{\le {}} + \frac{4 (C_{\cF} + C_{\cH})^2}{w}\sqrt{\frac{2}{n} \log\frac{1}{\delta}}\\
    &\phantom{\le {}} + \frac{2(C_{\cH} + C_{\cY})^2}{1-w}\sqrt{\frac{2}{n'}\log\frac{1}{\delta}}\\
    &\le \mathcal{O}_{p}\bigg(
      \mathfrak{R}_{n}(\cF)
      + \mathfrak{R}_{n}(\cH)
      + \mathfrak{R}_{n'}(\cH)
      + \frac{1}{\sqrt{n}} + \frac{1}{\sqrt{n'}}\bigg).
  \end{align*}
\end{proof}

\subsection{Calculation of the shrinkage term in Joint-RR}
\label{section:calc_UB_reg}
We will show the following proposition.
\begin{proposition}
  Under the assumption in Eq.~\eqref{eq:u_sufficient}, we have
  \begin{align*}
    &\min_{h\in L_U^2} J_w(f, h)\nonumber\\
    &= \MSE(f) + \text{const.}\nonumber\\
    &\phantom{=\ } + \underbrace{\frac{1-w}{w} \E[(f(X) - \E[f(X) \given U])^2]}_{\text{The shrinkage regularizer.}}
      \label{apxeq:gap_analysis}
  \end{align*}
  for any \(f\in L^{2}_{X}\) and any \(w \in (0, 1)\).
\end{proposition}

\begin{proof}
  On one hand, we have
  \begin{align*}
    &J_w(f, h)\\
    &= \frac{1}{w} \E[(f(X) - h(U))^{2}]
    + \frac{1}{1-w} \E[(h(U) - Y)^{2}]\\
    &= \frac{1}{w} \E[(f(X) - \E[f(X) \given U])^{2}]\\
    &\phantom{=\ } + \frac{1}{w} \E[(\E[f(X) \given U] - h(U))^{2}]\\
    &\phantom{=\ } + \frac{1}{1-w} \E[(h(U) - \E[Y \given U])^{2}]\\
    &\phantom{=\ } + \frac{1}{1-w} \E[(\E[Y \given U] - Y)^{2}]\\
    &= \frac{1}{w(1-w)} \bigg\{
      w \E[(h(U) - \E[Y \given U])^{2}] \\
    &\phantom{= \frac{1}{w(1-w)} \bigg\{\ }
      + (1-w)\E[(\E[f(X) \given U] - h(U))^{2}]
      \bigg\} + C_{1},
  \end{align*}
  where \(C_{1}\) is the remaining term that does not depend on \(h\).
  On the other hand, we have
  \begin{align*}
    &\E[(h(U) - h^{\circ}(U))^{2}]\\
    &= \E\bigg[\bigg( h(U) - w\E[Y \given U] - (1 - w) \E[f(X) \given U]\bigg)^{2}\bigg]\\
    &= \E[h(U)^{2}] + C_{2}\\
    &\phantom{=\ } -2w\E[h(U)\E[Y\given U]]\\
    &\phantom{=\ } + 2(1-w)\E[h(U)\E[f(X)\given U]]^{2}\\
    &= w\E[(h(U) - \E[Y\given U])^{2}] + C_{2}\\
    &\phantom{=\ } + (1-w)\E[(h(U) - \E[f(X)\given U])^{2}],
  \end{align*}
  where \(C_{2}\) is the remaining term that does not depend on \(h\).
  Thus, we have
  \begin{equation}
    J_w(f, h)
    = \frac{1}{w(1-w)} \E[(h(U) - h^{\circ}(U))^{2}] + C_{3},
  \end{equation}
  where \(C_{3}\) is the remaining term that does not depend on \(h\).
  This implies
  \begin{align*}
     h^{\circ} = \argmin_{h\in L^{2}_{U}} J_{w}(f, h),
  \end{align*}
  where \(h^{\circ}(u) \coloneqq w\E[Y \given U=u] + (1-w)\E[f(X) \given U=u]\).
  Finally, we can calculate the minimizer as
  \begin{align*}
     &\min_{h\in L^{2}_{U}} J_{w}(f, h)\\
     &= J_{w}(f, h^{\circ})\\
     &= \frac{1}{w} \E[(f(X) - \E[f(X) \given U])^{2}]\\
     &\phantom{=\ } + \frac{1}{w} \E[(\E[f(X) \given U] - h^{\circ}(U))^{2}]\\
     &\phantom{=\ } + \frac{1}{1-w} \E[(h^{\circ}(U) - \E[Y \given U])^{2}]\\
     &\phantom{=\ } + \frac{1}{1-w} \E[(\E[Y \given U] - Y)^{2}]\\
     &= \frac{1}{w} \E[(f(X) - \E[f(X) \given U])^{2}]\\
     &\phantom{=\ } + \frac{1}{w} w^{2}\E[(\E[f(X) \given U] - \E[Y\given U])^{2}]\\
     &\phantom{=\ } + \frac{1}{1-w} (1-w)^{2}\E[(\E[f(X) \given U] - \E[Y \given U])^{2}]\\
     &\phantom{=\ } + \frac{1}{1-w} \E[(\E[Y \given U] - Y)^{2}]\\
     &= \frac{1}{w} \E[(f(X) - \E[f(X) \given U])^{2}]\\
     &\phantom{=\ } + \E[(\E[f(X) \given U] - \E[Y\given U])^{2}]\\
     &\phantom{=\ } + \frac{1}{1-w} \E[(\E[Y \given U] - Y)^{2}]\\
     &= \lr(){\frac{1}{w} - 1} \E[(f(X) - \E[f(X) \given U])^{2}]\\
     &\phantom{=\ } + \E[(\E[f(X) \given U] - \E[Y\given U])^{2}]\\
     &\phantom{=\ } + \E[(f(X) - \E[f(X) \given U])^{2}]\\
     &\phantom{=\ } + \frac{1}{1-w} \E[(\E[Y \given U] - Y)^{2}]\\
     &= \frac{1-w}{w} \E[(f(X) - \E[f(X) \given U])^{2}]\\
     &\phantom{=\ } + \E[(f(X) - \E[Y\given U])^{2}]\\
     &\phantom{=\ } + \frac{1}{1-w} \E[(\E[Y \given U] - Y)^{2}]\\
     &= \MSE(f) + \frac{1-w}{w} \E[(f(X) - \E[f(X) \given U])^{2}] + C_{4}.
  \end{align*}
\end{proof}

\subsection{Discussion on the assumption}
\label{apx:assump}
So far, we have focused on the ideal case in which the conditional mean independence (Eq.~\eqref{eq:u_sufficient}) holds.
However, it may be difficult to exactly ensure the condition in practice.

Here, we relax Eq.~\eqref{eq:u_sufficient} by allowing the gap between the left-hand and right-hand sides to be potentially larger than zero
but bounded by \(c^{2}\) for some constant \(c \in (0, \infty)\).
We will show that (i) even the best possible method suffers an MSE of at least \(c^{2} / 2\)
in the worse-case within this scenario
while (ii) 2Step-RR suffers an MSE of at most \(c^{2} + o(1)\).

To see the claim (ii), note that we have already shown that \(\widetilde{f}\) converges to \(\E[\E[Y \given U] \given X = (\cdot)]\).
This implies that 2Step-RR method suffers an MSE of
\begin{align*}
  &\E[(\E[\E[Y \given U] \given X] - \E[Y \given X])^{2}] + o(1)\\
  &= \E[(\E[\E[Y \given U] - \E[Y \given U, X] \given X])^{2}] + o(1)\\
  &\le \E[(\E[Y \given U] - \E[Y \given U, X])^{2}] + o(1)\\
  &\le c^{2} + o(1)
\end{align*}
under the relaxed assumption.

To show the claim (i), let us first define the class of problem instances satisfying the relaxed assumption.
Fix any probability density function (p.d.f.\@) \(p(x)\) defined over \(\cX\).
For \(c \in (0, \infty]\), let \(\mathcal{P}_{c}\) denote the set of p.d.f.-s over \(\cX \times \cU \times \cY\) defined as follows.
Any p.d.f.\@ \(q(x, u, y)\) over \(\cX \times \cU \times \cY\) is a member of \(\mathcal{P}_{c}\)
if and only if, for random variables \((X, U, Y) \sim q(x, u, y)\),
\(X\) follows \(p(x)\) and Eq.~\eqref{eq:u_sufficient} is violated by \(c\) in \(L^2(p)\)-distance:
\begin{equation*}
  \E[(\E[Y \given U] - \E[Y \given U, X])^2] \le c^{2}.
\end{equation*}
In this setup, we will establish an lower bound of
\begin{equation*}
  E_\text{minimax}
  \coloneqq
  \inf_{\widehat{f}_{(\cdot)}}
  \sup_{q \in \mathcal{P}_{c}}
  \E[(\widehat{f}_{S_{X}, S_{Y}}(X) - \E[Y \given X])^2],
\end{equation*}
where
\(\widehat{f}_{(\cdot)}\) represents any estimator that uses mediated uncoupled data \((S_{X}, S_{Y})\) and produces a function from \(\cX\) to \(\cY\),
and the expectation is taken over \((S_{X}, S_{Y})\) and \((X, Y) \sim q(x, y)\).

  Define \(\rho\colon \mathcal{P}_{c}^2 \to [0, \infty)\) by
  \begin{align*}
    \rho(q_1, q_2)
    &\coloneqq \E[(\E[Y_1 \given X_{1} = X_{0}]\\
    &\phantom{\coloneqq\E[(\ } - \E[Y_2 \given X_{2} = X_{0}])^2]
  \end{align*}
  for any \(q_1, q_2 \in \mathcal{P}_{c}\),
  where \(X_{0} \sim p(x)\), \((X_{1}, U_1, Y_1) \sim q_1(x, u, y)\), and \((X_{2}, U_2, Y_2) \sim q_2(x, u, y)\),
  and they are all independent.

  We will use the following lemma.
  \begin{lemma}
    Suppose \(p(x)\) is a symmetric, centered density function over \(\cX\), i.e., \(p(x) = p(-x)\) and \(\int x p(x) \drm x = 0\).
    Then, we have
    \label{lem:mmlb_dist}
    \begin{align*}
      \inf_{\widehat{p}_{(\cdot)}} \sup_{q\in \mathcal{P}_{c}} \E\lr[]{\rho\lr(){\widehat{p}_{S_{X}, S_{Y}}, q}}
      &\ge \frac{1}{2} c^2.
    \end{align*}
  \end{lemma}

  Once we establish Lemma~\ref{lem:mmlb_dist}, we immediately obtain the following two propositions.

  \begin{proposition}
    \label{prop:mmlb_det}
    Suppose \(p(x)\) is a symmetric, centered density function over \(\cX\), i.e., \(p(x) = p(-x)\) and \(\int x p(x) \drm x = 0\).
    For any estimator \(\widehat{f}_{(\cdot)}\) that takes mediated uncoupled data as input and produces a function from \(\cX\) to \(\cY\), we have
    \begin{equation*}
      \sup_{p^* \in \mathbb{P}_{p, c}}
      \E\lr[]{\lr(){\widehat{f}_{S_{X}, S_{Y}}(X) - \E[Y \given X]}^2}
      \ge \frac{1}{2}c^2,
    \end{equation*}
    where \((X, Y) \sim p^*(x, y)\).
    This holds no matter how large \(n\) and \(n'\) are.
  \end{proposition}

  \begin{proposition}
    Suppose \(p(x)\) is a symmetric, centered density function over \(\cX\), i.e., \(p(x) = p(-x)\) and \(\int x p(x) \drm x = 0\).
    For any stochastic estimator \(\widehat{Y}_{(\cdot)}\) that takes mediated uncoupled data as input and produces a \(\cY\)-valued random variable depending on the test sample \(X\), we have
    \label{prop:mmlb_stc}
    \begin{align*}
      \sup_{q\in \mathcal{P}_{c}} \E\lr[]{\lr(){\E[\widehat{Y}_{S_{X}, S_{Y}} \given X] - \E[Y\given X]}^2}
      \ge \frac{1}{2} c^2,
    \end{align*}
    where \((X, Y) \sim q(x, y)\).
    This holds no matter how large \(n\) and \(n'\) are.
  \end{proposition}

  \begin{proof}[Proof of Proposition~\ref{prop:mmlb_stc} and Proposition~\ref{prop:mmlb_det}]
    If the statements were not to hold, it would contradict Lemma~\ref{lem:mmlb_dist}.
  \end{proof}

  We present our proof of Lemma~\ref{lem:mmlb_dist} in details
  since it provides an intuition about the problem with a concrete example.
  \begin{proof}[Proof of Lemma~\ref{lem:mmlb_dist}]
    The proof is based on Le Cam's method.
    We construct two p.d.f.-s within \(\mathcal{P}_{c}\)
    that are distant enough in terms of \(\rho\),
    whose corresponding mediated uncoupled data, however, have the identical distribution.

    For any \(\theta \in \Re\), let \(p(x, u, y; \theta)\) be the joint p.d.f.\@ of the variables defined by
    \begin{align*}
      X \sim p(x), \quad
      U \sim q(u), \quad
      Y = \theta X
    \end{align*}
    with any p.d.f.\@ \(q(u)\) over \(\cU\),
    where \(X\) and \(U\) are independent, and \(\theta \in \Re\).

    Take the two density functions
    \(p_1(x, u, y) \coloneqq p(x, u, y; c/\sigma)\) and \(p_2(x, u, y) \coloneqq p(x, u, y; -c/\sigma)\),
    where \(\sigma \coloneqq \sqrt{\Var[X]}\).
    Note that each of their parameters is the negation of the other.

    \(p_1\) and \(p_2\) both belong to \(\mathcal{P}_{c}\)
    because \(p_1(x) = p_2(x) = p(x)\), and
    both for \((X, Y, U) \sim p_1(x, u, y)\) and for \((X, Y, U) \sim p_2(x, u, y)\),
    \begin{align*}
                      &\E[(\E[Y \given U] - \E[Y \given U, X])^2]\\
                      &= \E[(0 + (c/\sigma) X)^2] \quad \text{(or \(\E[(0 - (c/\sigma) X)^2]\))}\\
                      &= c^2.
    \end{align*}

    Obviously, \(p_1(x, u) =  p(x)q(u) = p_2(x, u)\).
    Since \(p_1(y) = p_2(y)\) from the symmetry \(p(x) = p(-x)\), \(p_1(u, y) = q(u)p_1(y) = q(u)p_2(y) = p_2(u, y)\).
    Hence, the distribution of the mediated uncoupled data induced by \(p_1(x, u, y)\) and that by \(p_2(x, u, y)\)
    are identical to each other.
    On the other hand, \(p_1\) and \(p_2\) are \(2c\)-separated:
    \begin{align*}
      &\rho(p_1(x, u, y), p_2(x, u, y))\\
      &= \E[(\E[Y_1 \given X] - \E[Y_2 \given X])^2]\\
      &= \E[((c/\sigma) X + (c/\sigma) X)^2]\\
      &= (2c)^2.
    \end{align*}

    The argument above intuitively tells that
    it is impossible to distinguish distinct p.d.f.-s \(p_{1}\) and \(p_{2}\)
    only with the information given by the mediated uncouple data,
    and even the best possible guess would suffer loss proportionally to \(c^{2}\)
    in the worst case.
    More formally, by applying Le Cam's method, we obtain
    \begin{equation*}
      \inf_{\widehat{p}_{(\cdot)}} \sup_{q\in \mathcal{P}_{c}} \E\lr[]{\rho\lr(){\widehat{p}_{S_{n, n'}}(x, u, y), q(x, u, y)}}
      \ge \frac{1}{2} c^2.
    \end{equation*}
  \end{proof}

\subsection{Proof of Theorem~\ref{thm:ub}}
\label{apx:proof_ub}
See Appendix~\ref{apx:proof_ub_2}.
(This subsection is only a stub pointing to Appendix~\ref{apx:proof_ub_2} and will be removed in the next version.)

\subsection{Linear-in-parameter models}
\label{apx:linearmodels}
We consider the following regularized version of our method with linear-in-parameter models:
\begin{align*}
    \min_{\bm\theta\in\Re^{b_{\cF} + b_{\cH}}} \left[ \widehat{J}_w(f_{\bm\alpha}, h_{\bm\beta}) + \lambda \bm\theta^\top\bm\theta \right],
\end{align*}
where \(f_{\bm\alpha}(\bm x) \coloneqq \bm\alpha^\top \bm\phi(\bm x)\),
\(h_{\bm\beta}(\bm u) \coloneqq \bm\beta^\top \bm\psi(\bm u)\),
\(\bm\theta \coloneqq (\bm\alpha^\top, \bm\beta^\top)^\top\),
and \(\lambda \in (0, \infty)\).
Here,
\begin{align*}
    \widehat{J}_w(f_{\bm\alpha}, h_{\bm\beta})
    &= \frac{1}{nw}\sum_{i=1}^n (\bm\alpha^\top\bm\varphi(X_i) - \bm\beta^\top\bm\psi(U_i))^2\\
    &\phantom{=\ } + \frac{1}{n'(1-w)}\sum_{i=1}^{n'} (\bm\beta^\top\bm\psi(U'_i) - Y'_i)^2\\
    &= \bm\theta^\top\bm A\bm\theta - \bm b\bm\theta + \text{const.},
\end{align*}
where the matrices \(\bm A\) and \(\bm b\) are defined as
\begin{align*}
    &\bm A
    \coloneqq \frac{1}{nw}\sum_{i=1}^n\lr[]{
    \begin{pmatrix}
        \bm\varphi(X_i)\\
        -\bm\psi(U_i)
    \end{pmatrix}
    \begin{pmatrix}
        \bm\varphi(X_i)\\
        -\bm\psi(U_i)
    \end{pmatrix}^\top
    }\\
    &\phantom{\coloneqq\ }
      + \frac{1}{n'(1-w)}\sum_{i=1}^{n'}\lr[]{
    \begin{pmatrix}
        \bm 0\\
        \bm\psi(U_i)
    \end{pmatrix}
    \begin{pmatrix}
        \bm 0\\
        \bm\psi(U_i)
    \end{pmatrix}^\top
    },\\
    &\text{and}\quad
    \bm b \coloneqq \frac{1}{n'(1-w)}\sum_{i=1}^{n'}(\bm 0^\top, Y'_i\bm\psi(U'_i)^\top)^\top.
\end{align*}
Since \(A\) is positive semi-definite, we obtain the minimizer in a closed form by
\begin{align}
    \widehat{\bm\theta}
    &\coloneqq \argmin_{\bm\theta\in\Re^{b_{\cF} + b_{\cH}}} \left[ \hat{J}_w(\bm\theta) + \lambda \bm\theta^\top\bm\theta \right]\nonumber\\
    &= \argmin_{\bm\theta\in\Re^{b_{\cF} + b_{\cH}}} \left[ \bm\theta^\top(\bm A + \lambda\bm I)\bm\theta - \bm b\bm\theta \right]\nonumber\\
    &= (\bm A + \lambda\bm I_{b_{\cF} + b_{\cH}})^{-1}\bm b  \label{apxeq:theta_full_mat}
\end{align}
for any positive \(\lambda\), where \(\bm I_k\) denotes the \(k\)-by-\(k\) identity matrix.

Furthermore, using the block-wise matrix inversion formula,
we have
\begin{align}
    \widehat{\bm\alpha}
    &\coloneqq \bm M_1^{-1}\bm M_2 \widehat{\bm\beta},\quad\text{and}\quad
    \widehat{\bm\beta}
    (\bm M_3 - \bm M_2^\top\bm M_1^{-1}\bm M_2)^{-1} \bm b_1,  \label{apxeq:block_mat}
\end{align}
where
\begin{align*}
    \bm M_1 &\coloneqq \frac{1}{nw}\sum_{i=1}^n[\bm\varphi(X_i)\bm\varphi(X_i)^\top] + \lambda\bm I_{b_{\cF}},\\
    \bm M_2 &\coloneqq \frac{1}{nw}\sum_{i=1}^n[\bm\varphi(\bm x_i)\bm\psi(U_i)^\top],\\
    \bm M_3 &\coloneqq \frac{1}{nw}\sum_{i=1}^n[\bm\psi(U_i)\bm\psi(U_i)^\top]\\
           &+ \frac{1}{n'(1-w)}\sum_{i=1}^{n'}[\bm\psi(U'_i)\bm\psi(U'_i)^\top] + \lambda\bm I_{b_{\cH}},\\
    \bm b_1 &\coloneqq \frac{1}{n'(1-w)}\sum_{i=1}^{n'}[Y'_i\bm\psi(U'_i)].
\end{align*}
Eq.~\eqref{apxeq:block_mat} involves matrices of size at most \(\max(b_{\cF}, b_{\cH})\)-by-\(\max(b_{\cF}, b_{\cH})\),
which requires less computation resources in terms of both space and time compared to Eq.~\eqref{apxeq:theta_full_mat} involving inversion of a \((b_{\cF} + b_{\cH})\)-by-\((b_{\cF} + b_{\cH})\) matrix.

\subsection{Scatter Plots of MSEs for the Low-quality Image Classification}
\label{sec:MSE_plot_downsampled}
Figure~\ref{fig:MSE_downsampled} shows scatter plots of MSEs for the low-quality image classification.
We can see that the proposed methods outperform the naive method.
\begin{figure}[tp]
\centering
\subcaptionbox{MNIST. \label{fig:MNIST_downsampled_mse}}{\includegraphics[width=\columnwidth*23/48]{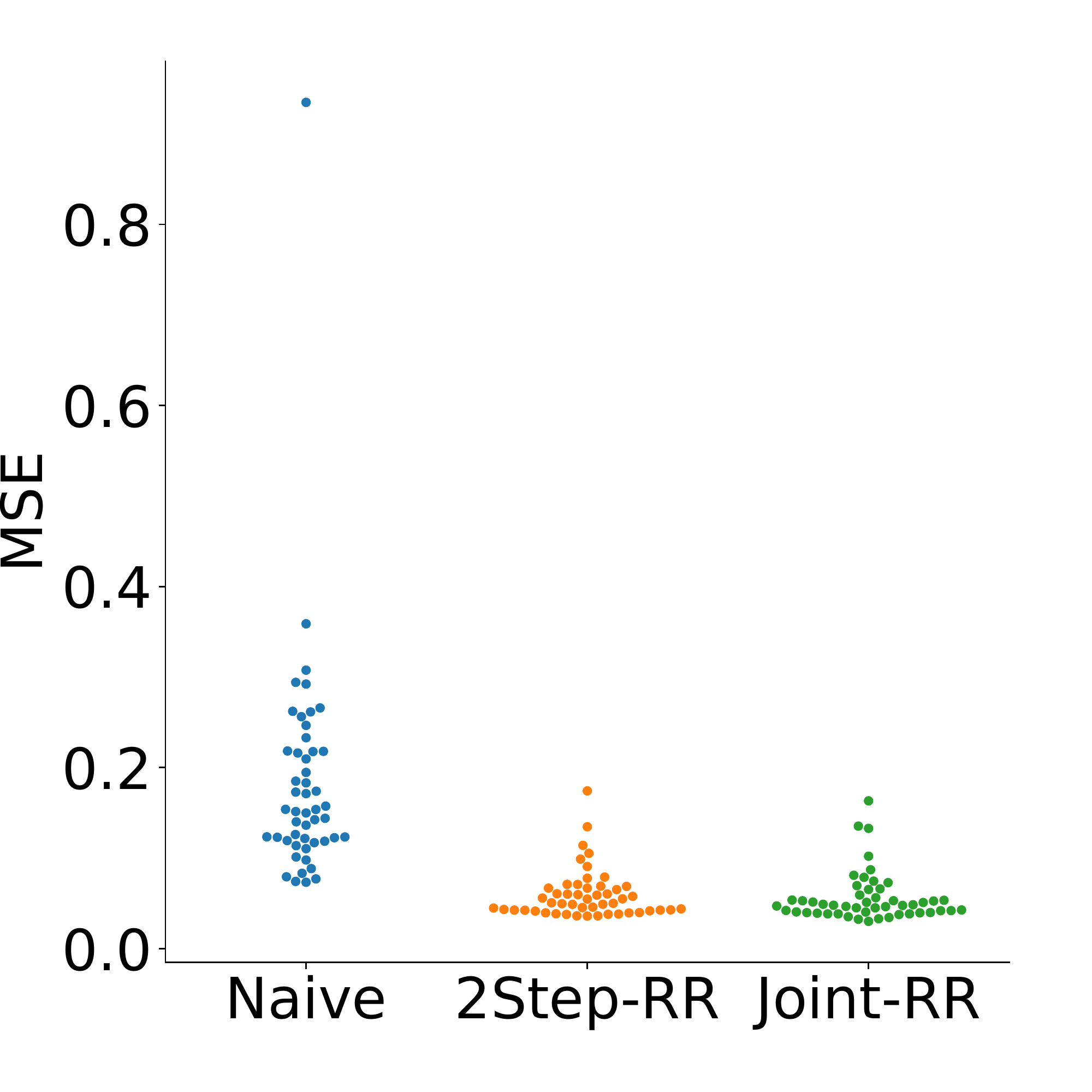}}
\subcaptionbox{Fashion-MNIST. \label{fig:FashionMNIST_downsampled_mse}}{\includegraphics[width=\columnwidth*23/48]{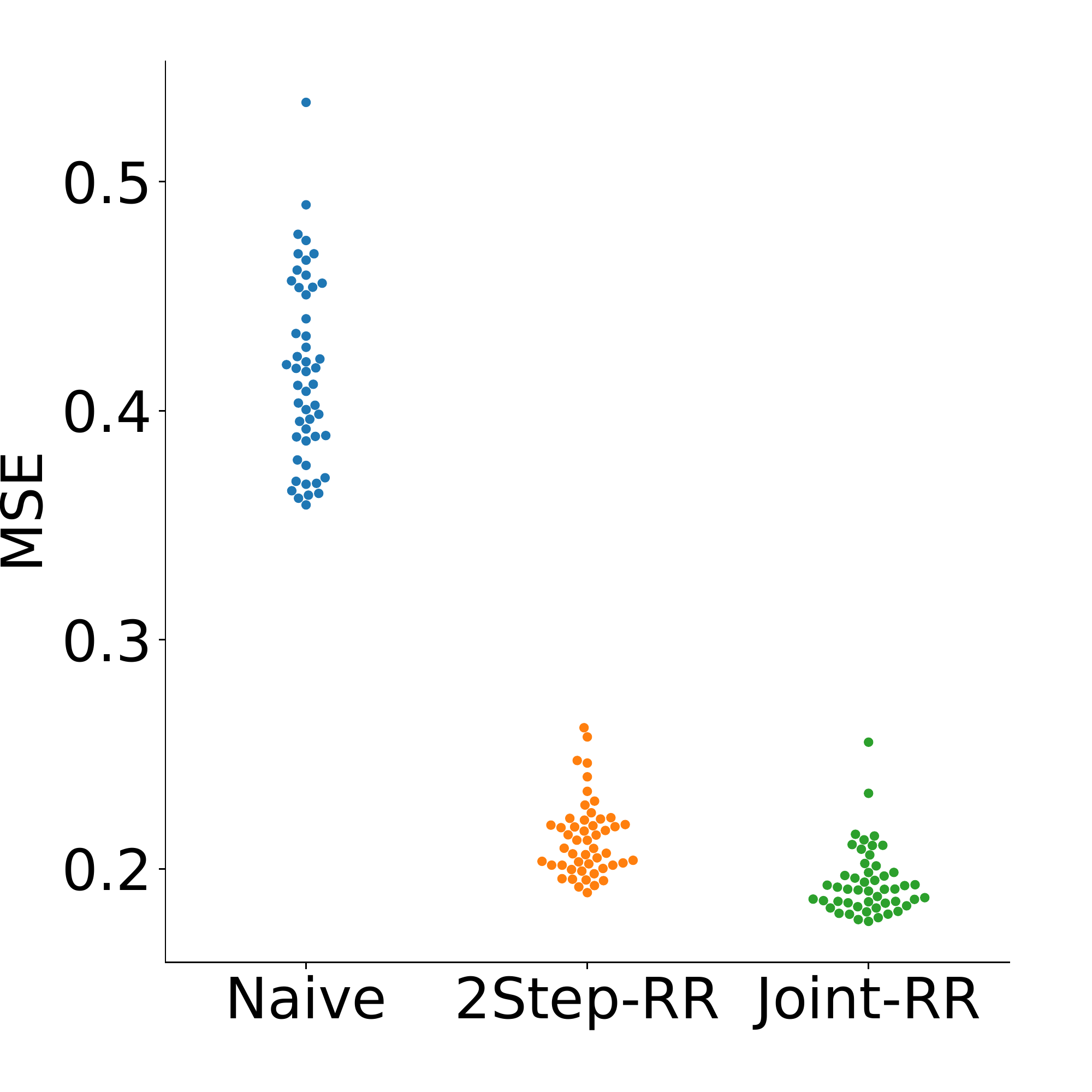}}
\subcaptionbox{CIFAR10. \label{fig:CIFAR10_downsampled_mse}}{\includegraphics[width=\columnwidth*23/48]{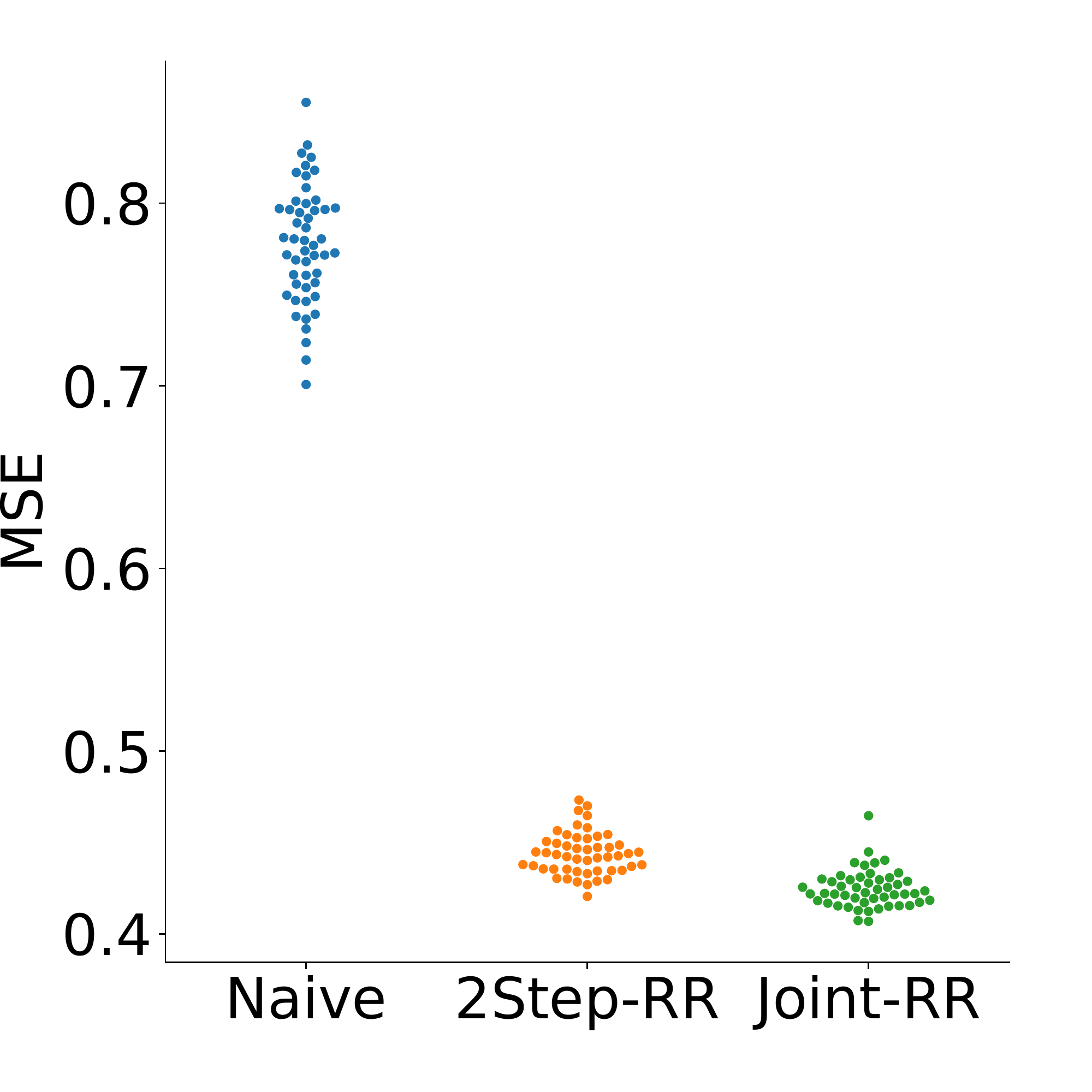}}
\subcaptionbox{CIFAR100. \label{fig:CIFAR100_downsampled_mse}}{\includegraphics[width=\columnwidth*23/48]{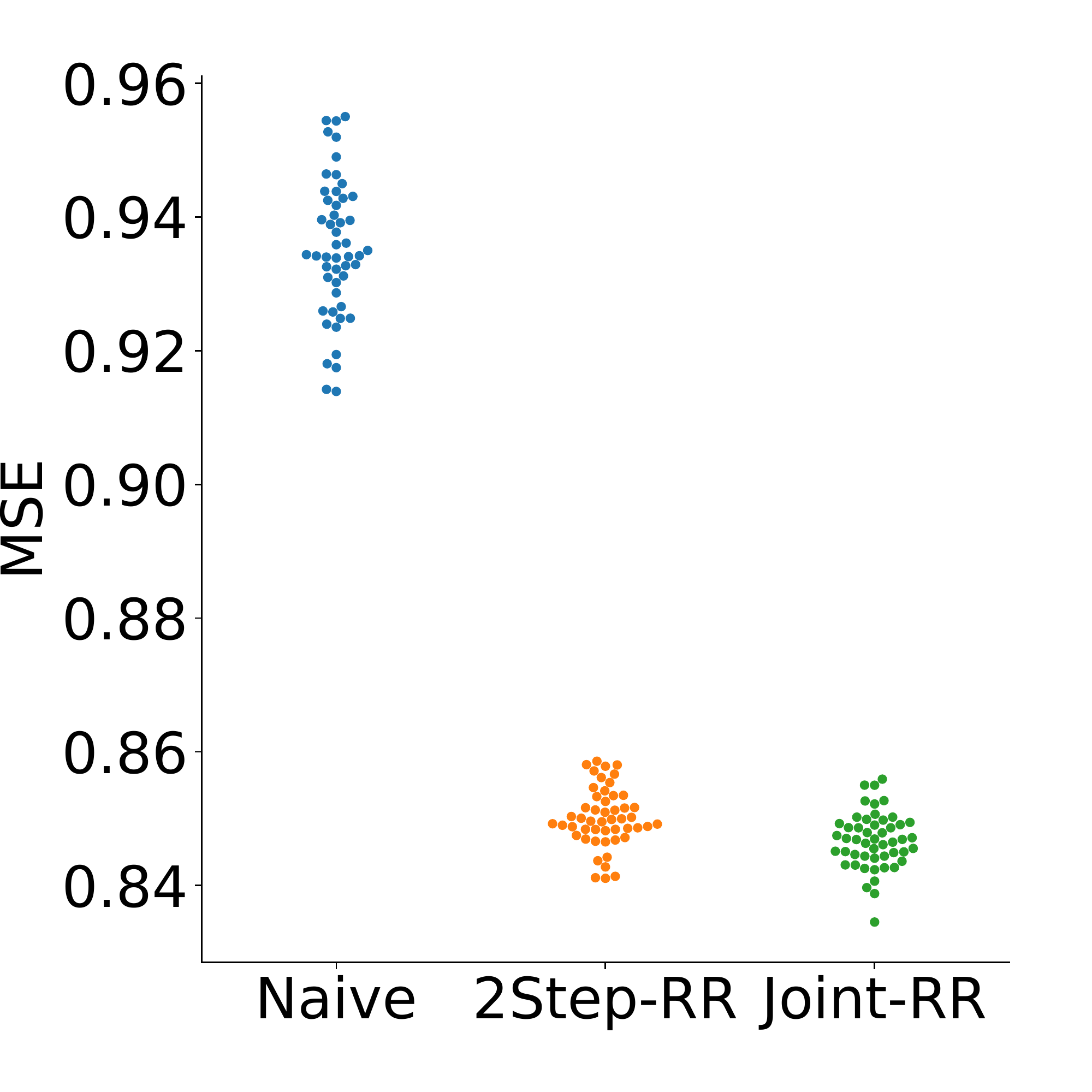}}
\caption{MSEs for the experiments on low-quality image classification.}
\label{fig:MSE_downsampled}
\end{figure}


\end{document}